%% file: main_paper.tex
\documentclass[11pt]{article}

\usepackage{chngcntr} 
\usepackage{amsmath}
\usepackage{enumerate}
\usepackage{fancyhdr}
\usepackage{color}
\usepackage{listings}
\usepackage{graphicx}
\usepackage{amssymb}
\usepackage{mathtools}

\usepackage{hyperref}
\usepackage{booktabs}

\usepackage[round]{natbib}
\input{macros}

\definecolor{lightgray}{gray}{0.5}

\usepackage[papersize={8.5in,11in},top=1in,bottom=0.75in,left=0.75in,right=0.75in]{geometry}
\usepackage[colorinlistoftodos]{todonotes}
\usepackage{enumitem}
\usepackage{tablefootnote}

\title{\textbf{An Improved Algorithm for Clustered Federated Learning }}

\author{Harshvardhan, Avishek Ghosh and Arya Mazumdar  \vspace{2mm} \\
Halıcıoğlu Data Science Institute  and Department of Computer Science\\
UC San Diego \\
\vspace{1.5mm}
email: \{hharshvardhan,a2ghosh,arya\}$@$ucsd.edu
}

\date{}
\begin{document}
\maketitle

\begin{abstract}
    In this paper, we address the dichotomy between heterogeneous models and simultaneous training in Federated Learning (FL) via a clustering framework. We define a new clustering model for FL based on the (optimal) local models of the users: two users belong to the same cluster if their local models are close; otherwise they belong to different clusters. A standard algorithm for clustered FL is proposed in \cite{ghosh_efficient_2021}, called \texttt{IFCA}, which requires \emph{suitable} initialization and the knowledge of hyper-parameters like the number of clusters (which is often quite difficult to obtain in practical applications) to converge. We propose an improved algorithm, \emph{Successive Refine Federated Clustering Algorithm} (\texttt{SR-FCA}), which removes such restrictive assumptions. \texttt{SR-FCA} treats each user as a singleton cluster as an initialization, and then successively refine the cluster estimation via exploiting similar users belonging to the same cluster. In any intermediate step, \texttt{SR-FCA} uses a robust federated learning algorithm within each cluster to exploit simultaneous training and to correct clustering errors. Furthermore, \texttt{SR-FCA} does not require any  \emph{good} initialization (warm start), both in theory and practice. We show that with proper choice of learning rate, \texttt{SR-FCA} incurs arbitrarily small clustering error. Additionally, we validate the performance of our algorithm on standard  FL datasets in non-convex problems like neural nets, and we show the benefits of \texttt{SR-FCA} over baselines\footnote{The code for all experiments is available at \href{https://github.com/harshv834/sr-fca}{https://github.com/harshv834/sr-fca}.}.
\end{abstract}

\input{intro}

\input{related_work}
\input{problem_setting}

\input{theoretical_results}

\input{experiments}

\bibliographystyle{abbrvnat}
\bibliography{references}

\appendix
\input{proof}

\end{document}

%% file: macros.tex
\usepackage{amssymb,amsmath,amsthm,dsfont}

\providecommand{\lin}[1]{\ensuremath{\left\langle #1 \right\rangle}}
\providecommand{\abs}[1]{\left\lvert#1\right\rvert}
\providecommand{\norm}[1]{\left\lVert#1\right\rVert}

  \providecommand{\R}{\mathbb{R}} 
  
  \DeclareMathOperator{\E}{{\mathbb E}}
  \providecommand{\E}[1]{{\mathbb E}\left.#1\right. }     

  \DeclareMathOperator*{\argmin}{arg\,min}

  \providecommand{\cD}{\mathcal{D}}

  \providecommand{\cN}{\mathcal{N}}
  \providecommand{\cO}{\mathcal{O}}

  \providecommand{\cW}{\mathcal{W}}

\definecolor{mydarkblue}{rgb}{0,0.08,0.45}

\newcommand{\ca}[1]{\mathcal{#1}}

\newcommand{\range}{\mathrm{rg}}

\usepackage{graphicx}
\usepackage{subfigure}
\usepackage{booktabs} 

\theoremstyle{plain}
\newtheorem{theorem}{Theorem}[section]
\newtheorem{proposition}{Proposition}
\newtheorem{lemma}{Lemma}
\newtheorem{corollary}{Corollary}
\theoremstyle{definition}
\newtheorem{definition}{Definition}
\newtheorem{assumption}{Assumption}
\newtheorem{remark}{Remark}

\usepackage{algorithm} 
\usepackage{algorithmic}
\usepackage{makecell}

%% file: intro.tex
\vspace{-4mm}
\section{Introduction}
\label{sec:intro}

In modern applications like recommendation systems, natural language processing, autonomous cars, image recognition, the size of data has exploded to such a point that distributed and parallel computing has become unavoidable. Furthermore, in many applications the data is actually stored at the edge---in users' personal devices like mobile phones and personal computers. Federated Learning, (FL) introduced in \citep{mcmahan2016communication,konevcny2016federated,mcmahan2017federated} is a large scale distributed learning paradigm aimed to exploit the machine intelligence in users' local devices. Owing to its highly decentralized nature, several statistical and computational challenges arise in FL, and in this paper, we aim to address one such challenge: heterogeneity.

The issue of heterogeneity is crucial for FL, since the data resides in users' own devices, and naturally no two devices have identical data distribution. There has been a rich body of literature in FL to address this problem of non iid data. A line of research assumes the \emph{degree of dissimilarity} across users are small, and hence focuses on learning a single global model~\citep{zhao2018federated,sahu-heterogeneous,li2018rsa,sattler-non-iid,mohri2019agnostic,karimireddy_scaffold_2020}. Note that learning a single model may not be sufficient in the situation where the degree of similarity is large or the users prefer to learn their personalized individual model. We direct the readers to two survey papers (and the references therein), \cite{li2020federated,kairouz2019advances} for a comprehensive list of papers on heterogeneity in FL.

As an alternative to the above, a new line of research in FL focuses on obtaining models personalized to individual users. For example  \cite{fedprox,li2021ditto} uses a regularization to obtaining individual models for users and the regularization ensures that the local models stay close to the global model. Another line of work poses the heterogeneous FL as a meta learning problem~\citep{chen2018federated,jiang2019improving,fallah_personalized_2020,fallah2020convergence}. Here, the objective is to first obtain a single global model, and then each device run some local iterations (fine tune) the global model to obtain their local models. Furthermore \cite{collins-exploit} exploits shared representation across users by running an alternating minimization algorithm and personalization. Note that all these personalization algorithms, including meta learning, work only when the local models of the users' are close to one another.

On the other spectrum, when the local models of the users may not be close to one another, \cite{sattler2019clustered,mansour2020three,ghosh_efficient_2021} propose a clustering framework, where the objective is to obtain individual models for each cluster. Note that \cite{sattler2019clustered} uses a centralized clustering scheme, where the center has a significant amount of compute load, which is not desirable for FL. Furthermore, it uses a top-down approach using cosine similarity metric between gradient norm as optimization objective. Also, the theoretical guarantees of  \cite{sattler2019clustered} are limited. Further, in \cite{duan2020fedgroup}, a data-driven similarity metric is  used extending the cosine similarity and the framework of \cite{sattler2019clustered}. Moreover, in \cite{mansour2020three}, the authors propose algorithms for both clustering and personalization. However, they provide guarantees only on generalization, not iterate convergence. In  ~\cite{mocha} the job of  multi-task learning is framed as clustering where a regularizer in the optimization problem defines clustering objective.

Very recently, in \cite{ghosh_efficient_2021}, a framework of \textit{Clustered Federated Learning} is analyzed and an iterative algorithm, namely \texttt{IFCA}, is proposed that attains (exponential) convergence guarantees under \emph{suitable} initialization. Moreover, in \cite{ghosh_efficient_2021}, all the users are partitioned into a fixed and known number of clusters, and the users' in each cluster have identical data distribution. Note that, the convergence guarantee of \texttt{IFCA} depends crucially on \emph{suitable} initialization condition (or warm start), which is also impractical for practical applications. Furthermore, it is discussed in the same paper that the knowledge about the number of clusters is quite non-trivial to obtain in applications (see \citep[Section 6.3]{ghosh_efficient_2021}). Moreover, since the machines inside a cluster have data from a same distribution, the local models of all the machines in a cluster are identical (we formalize this in Section~\ref{sec:pblm_setting}). This is a fairly strong assumption, since in FL applications, users who belong to same cluster, may have similar data. However, the data distribution may not be identical, and so, their local models are close, but not identical. 

Following \texttt{IFCA}, a number of papers attempt to extend the federated clustering framework. For example, in \cite{fedsoft}, a  soft-clustering version of \texttt{IFCA} was proposed, where each data on a device can belong (probabilistic-ally) to different cluster. In \cite{multi_center}, an \texttt{IFCA} inspired algorithm is proposed that uses neuron matching as distance metric. Moreover, the soft clustering idea is also discussed in \cite{soft_cluster}. To the best of our knowledge, although the above-mentioned papers extend \texttt{IFCA} in certain directions, the crucial shortcomings of \texttt{IFCA}, namely \emph{good} initialization and the knowledge of the number of clusters remain unanswered.

In this paper, we propose a clustering framework for Federated Learning that overcomes the above-mentioned shortcomings. Specifically, we propose and analyze an algorithm, namely \emph{Successive Refine Federated Clustering Algorithm} (\texttt{SR-FCA}), which iteratively estimates an refines the cluster identities of the machines. We show that, \texttt{SR-FCA} obtains arbitrary small clustering error. One attractive feature of \texttt{SR-FCA} is that it does not require the knowledge of the number of clusters apriori (instead it requires a weak condition on the minimum size of the cluster, as explained in Section~\ref{sec:algo}). Moreover, \texttt{SR-FCA} works with \emph{arbitrary} initialization, which is a major theoretical as well as practical novelty over existing literature. Furthermore, in \texttt{SR-FCA} we remove the requirement that all machines belonging to the same cluster possesses same local model.

To be precise, we define a novel clustering structure (see Definition~\ref{assumption:clustering}), based on the local models on the worker machines\footnote{Throughout the paper, we use nodes, users, machines, workers synonymously to denote the compute nodes in FL.}. Classically, clustering is defined in terms of distribution from which the machines sample data. However, in a federated framework, it is common to define a heterogeneous framework such as clustering in terms of other discrepancy metric; for example in \cite{mansour2020three}, a metric that depends on the local loss is used.

\textbf{Distance Metric:} In this paper, we use a distance metric across workers' local model as a discrepancy measure and define a clustering setup based on this. Our distance metric may in general  include non-trivial metric like Wasserstein distance, $\ell_q$ norm (with $q \geq 1)$ that captures desired practical properties like permutation invariance and sparsity for (deep) neural-net training. For our theoretical results, we focus on strongly convex and smooth loss for which $\ell_2$ norm of iterates turns out to be the natural choice. However, for non-convex neural networks on which we run most of our experiments, we use a cross-cluster loss metric. For two clients $i,j$, we define their cross-cluster loss metric as the average of the cross entropy loss of client $i$ on the model of client $j$ and the cross entropy loss of client $j$ on the model of client $i$. If this metric is low, we can use the model of client $i$ for client $j$ and vice-versa, implying that the clients are similar.

With the above discrepancy metric, we put the machines in same cluster if their local models are close -- otherwise they are in different clusters. We emphasize that, we do not necessarily require the local models in a cluster to be identical; a closeness condition is sufficient. Moreover, the clustering identities of the workers are unknown apriori, and \texttt{SR-FCA} iteratively learns them. We now list our contributions.
\subsection{Our Contributions}
\subsubsection{Algorithmic}
We introduce a new clustering framework based on local user models and propose an iterative clustering algorithm, namely \texttt{SR-FCA}. Our algorithm starts with simple pairwise distance based clustering, and refine those estimates (and merge in necessary) over multiple rounds. We observe that the successive refinement step exploits collaboration across users in the same cluster, and reduces the clustering error. In particular, we use the first order gradient based robust FL algorithm of \cite{pmlr-v80-yin18a} for federation. We require a robust algorithm because we treat the wrongly clustered machines as outliers. However, we do not throw the outliers away like \cite{pmlr-v80-yin18a}; rather we reassign them to their closest cluster.

When the loss is strongly convex and smooth, and  $\textsf{dist}(.,.)$  is $\ell_2$ norm, we show that, the mis-clustering error in the first stage of \texttt{SR-FCA} is given by $\mathcal{O}(md \exp(-n/\sqrt{d})$ (Theorem~\ref{thm:init}), where $m$, $n$ and $d$ denote the number of worker machines, the amount of data in each machine and the dimensionality of the problem respectively. Moreover, successive stages of \texttt{SR-FCA} further reduce the mis-clustering error by a factor of $\mathcal{O}(1/m)$ (Theorem~\ref{thm:refine}), and hence yields arbitrarily small error. In practice we require a very few refinement steps (we refine at most twice in experiments, see Section~\ref{sec:experiments}). 

Furthermore, we compare our results with \texttt{IFCA} both theoretically and experimentally. We notice that the requirement on the separation of clusters is quite mild for \texttt{SR-FCA}. We only need the separation to be $\Tilde{\Omega}(\frac{1}{n})$\footnote{Here, $\Tilde{\Omega}$ hides logarithmic dependence}. On the other hand, in certain regimes, \texttt{IFCA} requires a separation of $\Tilde{\Omega}(\frac{1}{n^{1/5}})$, which is a much stronger requirement. 

As a by-product of \texttt{SR-FCA}, we also obtain an appropriate loss minimizer for each cluster, defined in Eq.~\eqref{eq:wstar}---which in conjunction with our clustering model, is a reasonable good approximation for all the machines in that cluster (see Theorem~\ref{thm:convergence}). The clustering estimates are obtained by leveraging federation across similar users in a cluster. We notice that the statistical error we obtain here is $\Tilde{\mathcal{O}}(1/\sqrt{n})$, which is weaker than \texttt{IFCA} (statistical error of \texttt{IFCA} is $\Tilde{\mathcal{O}}(1/n)$; see \cite[Theorem 2]{ghosh_efficient_2021}). This weaker rate can be thought as the price of random initialization. For \texttt{IFCA}, a \emph{good} initialization implies that only a very few machines are mis-clustered, which was crucially required to obtain the $\Tilde{\mathcal{O}}(1/n)$ rate. But, for \texttt{SR-FCA}, we do not have such guarantees to begin with, and we necessarily take a union bound on \emph{all} machines, which results in a weaker statistical error.
\vspace{-2mm}
\subsubsection{Technical Novelty}
\vspace{-1mm}
A key requirement in any clustering problem is \emph{suitable} initialization. However, \texttt{SR-FCA} removes this requirement completely, and allows the worker machines to start arbitrarily and run some number of local iterations. We show that provided the loss function is strongly convex and smooth, and the problem is \emph{well-separated}, pairwise distance based clustering of these local iterates provide a \emph{reasonably} good initialization, and our algorithm refines this iteratively. However, this initialization imposes further complications---for example, an original cluster can split into multiple clusters or some machines may be unassigned to any clusters. To handle all these issues, we crucially leverage (a) sharp generalization guarantees for strongly convex losses with subexponential gradients and (b) robustness property of the trimmed mean estimator (of \cite{pmlr-v80-yin18a}).

\subsubsection{Experiments}
We implement \texttt{SR-FCA} on wide variety of datasets including synthetic, rotated or inverted MNIST, CIFAR10, FEMNIST and Shakespeare~\cite{leaf}). With $\ell_2$ distance metric for synthetic and cross-cluster loss for the rest, we observe that the test performance of  \texttt{SR-FCA} outperforms three baselines---(a) global (one model for all machines) and (b) local (one model per machine), (c)\texttt{IFCA}. Further, on synthetic and simulated datasets, \texttt{SR-FCA} recovers the correct clustering $\mathcal{C}^\star$. On real datasets, \texttt{SR-FCA} can figure out the correct number of clusters and thus outperforms \texttt{IFCA} by around 4\%.

%% file: related_work.tex
\section{Related Work}
Data heterogeneity is a major challenge in FL, and various frameworks have been proposed in the recent past to address this. Where \cite{zhao2018federated,li2018rsa,mohri2019agnostic,karimireddy_scaffold_2020} uses \emph{degree of dissimilarity} to model the data heterogeneity, \cite{chen2018federated,fallah_personalized_2020,fallah2020convergence} uses meta-learning to achieve the same. 

Clustering is a canonical way to model heterogeneous data, and there is a significant interest in the community. Starting with the centralized cosine similarity based clustering of \cite{sattler2019clustered}, and multi-task based approach of \cite{mocha}, we have the decentralized algorithm (\texttt{IFCA}) of \cite{ghosh_efficient_2021}. Furthermore, there are several variations of \texttt{IFCA}; for example, in \cite{fedsoft,soft_cluster} soft-clustering versions of \texttt{IFCA} are proposed, in \cite{multi_center}, the authors use neuron matching as distance metric.

Apart from this, several personalized (local tuning) models for FL use ideas from clustering. For example, \cite{fallah_personalized_2020} proposes Hypcluster algorithm, which is similar to \texttt{IFCA}. Moreover, \cite{zhang2021parameterized} use personalized method which uses base clustering for knowledge transfer. Furthermore, \cite{hierarchy} use hierarchical top-down clustering and then propagate local update steps down the hierarchy for better trained model. 

%% file: problem_setting.tex
\section{Problem Formulation}
\label{sec:pblm_setting}
We have $m$ machines that are partitioned into disjoint clusters,
denoted by the clustering map $\mathcal{C}^\star : [m] \to [C]$\footnote{We denote $[n]\equiv \{1, 2, \dots, n\}$ for the rest of the paper.}, where $C$ is the (unknown) number of clusters. 
For any clustering map $\mathcal{C}'$, let $\range(\mathcal{C}')$ denote the range of the map.
Each node $i\in[m]$  contains $n_i \geq n$
data points $\{z_{i,j}\}_{j=1}^{n_i}$ sampled from a distribution $\cD_i$. 

We define $f(\cdot ;z):\cW  \to \R$ as the loss function for the sample $z$, where $\cW \subseteq \R^d$.
Here, $\cW$ is a closed and convex set with diameter $D$. We now define the population loss, $F_i:\cW\to\R^d$, and its minimizer, $w_i^\star$,  for each node $i \in [m]$.
\begin{align*}
    F_i(w) = \E_{z \sim \cD_i}[f(w, z)],\quad
    w_i^\star = \min_{w \in \cW} F_i(w)
\end{align*}
The clustering $\mathcal{C}^\star$ is based on the population minimizers of nodes $w_i^\star$. This is defined as:
\begin{definition}[Clustering Structure]\label{assumption:clustering} For a distance metric $\mathsf{dist}(.,.)$, the local models satisfy
\begin{equation}
\begin{aligned}
    &\max_{i,j:\mathcal{C}^\star(i) = \mathcal{C}^\star(j)} \mathsf{dist}(w_i^\star,w_j^\star) \leq \epsilon_1 ,\quad \quad \min_{i,j:\mathcal{C}^\star(i)\neq \mathcal{C}^\star(j)} \mathsf{dist}(w_i^\star,w_j^\star) \geq \epsilon_2,
\end{aligned}
\end{equation}
where $\epsilon_1,\epsilon_2$, are non-negative constants with $\epsilon_2 > \epsilon_1$.
\end{definition}
\begin{remark}
The above definition serves as a measure of heterogeneity in our system wrt $\mathsf{dist}(.,.)$, with those inside having similar population minimizers and those in different clusters having different population minimizers.
\end{remark}
\begin{remark} [Relaxation over \texttt{IFCA}]
The above structure allows the population minimizers inside clusters to be close, but not necessarily equal, as opposed to ~\cite{ghosh_efficient_2021}(i.e., \texttt{IFCA} assumes $\epsilon_1=0$).
\end{remark}

In practice, we have access to neither $F_i$ nor $w_i^\star$, but only the sample mean variant of the loss, the empirical risk, $f_i(w) = \frac{1}{n_i}\sum_{j=1}^{n_i} f(w,z_{i,j})$, and its derivatives , for each node $i \in [m]$.

If $G_c \equiv \{ i : i\in [m] , \mathcal{C}^\star(i) = c \}$, denotes the set of nodes in cluster $c$.
We can then define the population loss per cluster $c\in [C]$ as $\ca{F}_c$
\begin{align}\label{eq:wstar}
    \ca{F}_c(w) = \frac{1}{\abs{G_c}}\sum_{i \in G_c}F_i(w), \,\,\, \omega^\ast_c = \argmin_{w \in \cW} \ca{F}_c(w), \,\, \forall c \in [C].
\end{align}
Our final goal is to find an appropriate population loss minimizer for each cluster $c\in [C]$, i.e., $\omega^\ast_c$. Obtaining this involves several layers of complexity: we need to find the correct clustering $\mathcal{C}^\star$ and recover the minimizer of each cluster.
The main difficulties in this setting are: (a) the number of clusters is not known beforehand. This prevents us from using most clustering algorithms like k-means; and (b) The clustering depends on $w_i^\star$ which we do not have access to. We can estimate $w_i^\star$ by minimizing $f_i$, however when $n$, the minimum number of data points per node, is small,  this estimate may be very far from $w_i^\star$.

In spite of the inherent complexity of our problem, we can overcome it by utilizing federation.
Note that estimating  $w_i^\star$ is difficult if each node does
not have enough data points. This apparent difficulty can be mitigated if instead we try to estimate $\omega_c^\star$, the population minimizer for each cluster. For each cluster, we can hope
that nodes in that cluster work together to improve our estimate of $\omega_\mathcal{C}^\star$. But, this requires us to know the clustering. The circular nature of our problem
implies that we need to use an alternating algorithm, which estimates $\omega_c^\star$ in one step and then tries to cluster based on the estimates of $\omega_c^\star$. This forms the main idea of our clustering algorithm.
\section{Algorithm-\texttt{SR-FCA}}
\label{sec:algo}
In this section, we formally present out clustering algorithm, namely \texttt{SR-FCA}. It starts with a subroutine \texttt{ONE\_SHOT} which outputs an initial clustering.  \texttt{SR-FCA} then successively calls the \texttt{REFINE()} subroutine to improve the clustering.

In each step of \texttt{REFINE()}, we first estimate the cluster iterates for each cluster. Then, based on these iterates we regroup all the nodes using \texttt{RECLUSTER()} and then if required merge the resulting clusters, using \texttt{MERGE()}. We crucially require an initial clustering to start \texttt{REFINE()}, and \texttt{ONE\_SHOT} provides that when \texttt{REFINE()} is called for the first time.

Since our goal is to end up with the correct clustering $\mathcal{C}^\star$ and its cluster models, it is important to understand how similar any arbitrary clustering  $C'$ (for instance the ones generated by \texttt{ONE\_SHOT} and \texttt{REFINE}) is to the original clustering $\mathcal{C}^\star$. To that end, consider a cluster $c \in \range(C')$. Then, we can first define its label with respect to $\mathcal{C}^\star$.
\begin{definition}[Cluster label]
\label{def:cluster_label}
    We define $c' \in [C]$, as the cluster label of cluster $c \in \range(C')$ if the majority of nodes in $c$ are originally from $c'$.
\end{definition}
This definition allows us to map each cluster  $c \in \range(C')$ to a cluster $c$ in $\mathcal{C}^\star$. Using the cluster label $c'$, we can define the impurities in cluster $c$ as the nodes which did not come from $c'$. Therefore, if a cluster has cluster label $c'$ and has very low fraction of impurities, it can serve as a good proxy for the original cluster $c'$ in $\mathcal{C}^\star$.

We now explain the different subroutines.

\begin{algorithm}[t!]
    \caption{SR-FCA}\label{alg:SR_FCA}
    \begin{algorithmic}
        \STATE {\bfseries Input:}  Threshold
        $\lambda$, Size parameter $t$
        \STATE {\bfseries Output:} Clustering $C_R$

        \STATE $C_0 \gets$ \texttt{ONE\_SHOT(}$\lambda$, $t$\texttt{)}\\

        \FOR{$r = 1$ to $R$}
        \STATE $C_r \gets$ \texttt{REFINE(}$C_{r-1}, \lambda$\texttt{)}
        \ENDFOR
        \STATE \underline{\texttt{ONE\_SHOT(}$\lambda, t$\texttt{)}}
        \FOR{all $i$ clients in parallel }
        \STATE $w_i \gets $ Train local model for client $i$ for $T$ steps
        \ENDFOR
        \STATE $G \gets$ Graph with $m$ vertices and no edges
        \FOR{all pairs of clients $i,j\in [m], i\neq j $}
        \STATE Add edge $(i,j)$ to the graph $G$ if $\mathsf{dist}(w_{i,T},w_{j,T}) \leq \lambda$
        \ENDFOR
        \STATE $C_0 \gets$ Obtain clusters from graph $G$ with size $\geq t$ by correlation clustering of \cite{Bansal02correlationclustering}.
        \STATE \underline{\texttt{REFINE($C_{r-1}, \lambda$)}}
        \FOR{all clusters $c \in C_{r-1}$}
        \STATE $w_c \gets$ TrimmedMeanGD()
        \ENDFOR
        \STATE $C_r' \gets $\texttt{RECLUSTER(} $C_{r-1}$\texttt{)}
        \STATE $C_{r} \gets$ \texttt{MERGE(}$C_r',\lambda,t$\texttt{)}
    \end{algorithmic}
\end{algorithm}

\subsection{\textbf{\texttt{ONE\_SHOT()}}}

For our initial clustering, we create edges between nodes based on the distance between their locally trained models if  $\mathsf{dist}(w_{i}, w_{j}) \leq \lambda$, for a threshold $\lambda$ and then obtain clusters from this graph by correlation clustering of \cite{Bansal02correlationclustering}.
We only keep the clusters which have at least $t$ nodes.

If our locally trained models $w_i$, were close to their population minimizers $w_i^\star$, for all nodes $i\in [m]$, then choosing the threshold $\lambda \in (\epsilon_1, \epsilon_2)$,
we obtain edges between only clients which were in the same cluster in $\mathcal{C}^\star$. However, if $n$, the number of local datapoints is small, then our estimates of
local models $w_i$ might be very far from their corresponding $w_i^\star$ and we will not be able to recover $\mathcal{C}^\star$.

However, $\mathcal{C}_0$ is still a good clustering  if it satisfies these requirements: (a) if every cluster in $\range(\mathcal{C}^\star)$ has a good proxy (in the sense of Definition~\ref{def:cluster_label}) in $\range(\mathcal{C}_0)$ , and (b) each cluster in $\mathcal{C}_0$ has at most a small fraction ($ <\frac{1}{2}$) of impurities in it.
These requirements imply that the clusters in $\mathcal{C}_0$ are mostly ``pure" and represent all clusters in $\mathcal{C}^\star$. If the loss is structured, as shown in Theorem~\ref{thm:init}, we can obtain low mis-clustering error after \texttt{ONE\_SHOT}.

\subsection{\textbf{\texttt{REFINE():}}}

\textbf{Subroutine TrimmedMeanGD()}:
The main problem with \texttt{ONE\_SHOT()}, namely, small $n$, can be mitigated if we use federation.
Since, $\mathcal{C}_0$ has atleast $t$ nodes per cluster, training a single model for each cluster will utilize $tn$ datapoints, making the estimation more accurate.
However, the presence of impurities in a cluster can hamper this, motivating the use of a robust training algorithm, in this case, TrimmedMean \citep{pmlr-v80-yin18a}.

This subroutine is similar to FedAvg \citep{mcmahan2016communication}, but instead of taking the average of local models, we take the coordinate-wise trimmed mean. We use $\beta < \frac{1}{2}$ to
define the level of trimming and tune it in our experiments.

We end up with a trained model for each cluster as an output of this subroutine. Since these models are better estimates of their population risk minimizers,
we can use them to improve $\mathcal{C}_0$.

\textbf{Subroutine RECLUSTER()}
The purpose of this subroutine is to reduce the impurity level of each cluster in $\mathcal{C}_0$.
This is done by assigning each client $i$ to its nearest cluster $c$ in terms of $\mathsf{dist}(w_c, w_i)$.
Since $w_c$ are better estimates, we hope that the each impure node will go to a cluster with its actual cluster label. If we have a cluster in $\range(\mathcal{C}^\star)$ which does not have a good proxy in $\range(\mathcal{C}_0)$, then the nodes of this cluster will always be impurities.

\textbf{Subroutine MERGE():}
Even after removing all impurities from each cluster, we can still end up with more than 1 cluster having the same cluster label.
In $\mathcal{C}^\star$, these form the same cluster, thus they should be merged. 
If two clusters were originally from the same cluster in $\mathcal{C}^\star$, then their learned models should be very close. Therefore, we use an approach similar to \texttt{ONE\_SHOT},
where instead of distances between client models, we take distances between cluster models, construct edges and then merge the clusters.

Note that the complete algorithm with formal descriptions of all  these subroutines can be found in Appendix~\ref{sec:complete_algo}.
\subsection{Discussion}
\texttt{SR-FCA} uses a bottom-up approach to construct and refine clusters. The initialization in \texttt{ONE\_SHOT} is obtained by distance-based thresholding on local models. These local models are improper estimates of their population minimizers due to small $n$, causing $\mathcal{C}_0 \neq \mathcal{C}^\star$. However, if $\mathcal{C}_0$ is not very bad, i.e., each cluster has $<\frac{1}{2}$ impurity fraction and all clusters in $\mathcal{C}^\star$ are represented, we can refine it.

\texttt{REFINE()} is an alternating procedure, where we first estimate cluster centers from impure clusters. Then, we \texttt{RECLUSTER()} to remove the impurities in each cluster and then \texttt{MERGE()} the clusters which should be merged according to $\mathcal{C}^\star$. Note that as these steps use cluster estimates which are more accurate, they should have smaller error 

This iterative procedure should recover one cluster for each cluster in $\mathcal{C}^\star$, thus obtaining the number of clusters and every cluster should be pure, so that $\mathcal{C}^\star$ is exactly recovered. 

Note that the \texttt{TrimmedMeanGD} procedure also returns iterates, however, these may not have the best performance. Once we have recovered $\mathcal{C}^\star$, we can run a FL algorithm inside each cluster if we need better cluster iterates.

In the next section, we will provide theoretical justification for several of our claims and establish the probability of clustering error and convergence rates for the cluster iterates obtained by \texttt{TrimmedMeanGD}.

%% file: theoretical_results.tex
\section{Theoretical Guarantees}
\label{sec:theory}
In this section, we obtain the convergence guarantees of \texttt{SR-FCA}. For theoretical tractability,  we restrict to the setting where the \textsf{dist}$(.,.)$ is the euclidean ($\ell_2$) norm. However, in experiments (see next section), we remove this restriction and work with other \textsf{dist}$(.,.)$ functions. Here, we show an example where $\ell_2$ norm comes naturally as the \textsf{dist}$(.,.)$ function.
\begin{proposition}\label{prop:lin_reg}
Suppose that there are $m$ clients, each with a local model $w_i^\star \in \R^d$ and its datapoint $(x,y_i) \in \R^d \times \R$ is generated according to $y_i = \lin{w_i^\star, x} + \epsilon_i$. If $x \sim \cN(0,I_d)$ and $\epsilon_i \stackrel{i.i.d}{\sim}\cN(0,\sigma^2)$, then $\E_x[KL(p(y_i | x)||p(y_j| x))] = \frac{d}{2\sigma^2} \norm{w_i - w_j}^2$.
\end{proposition}
Hence, we see that minimizing a natural distance (KL divergence) between the conditional distribution $y|x$ for different clients is equivalent to minimizing the $\ell_2$ norm of the underlying local models.

Our goal here is to recover both the clustering and clustering iterates. We first quantify the probability of not recovering the original clustering, i.e., $C_r\neq C^\star$. Here and subsequently, two clusters being not equal means they are not equal after relabeling (see Definition~\ref{def:cluster_label}). 

\begin{remark}
Recall that in Algorithm~\ref{alg:SR_FCA}, we choose $\lambda$ and $t$ as (any) input parameter to the algorithm. However, for the guarantees of this section to hold, we require $\lambda \in (\epsilon_1, \epsilon_2)$ and $t \leq c_{\min}$, where $c_{\min}$ is the minimum size of the cluster. We emphasize that, in practice (as shown in the experiments), we treat $
\lambda$ and $t$ as hyper-parameters and obtain them by tuning. Hence, we do not require the knowledge of $\epsilon_1, \epsilon_2$ and $c_{\min}$.
\end{remark}
\begin{remark}
Although in Algorithm~\ref{alg:SR_FCA}, we use correlation clustering for finding the initial clusters, in theory, we restrict ourselves to finding cliques only. Note that if there are cliques in the graph, then correlation clustering will identify them. 
\end{remark}

We start with reviewing the standard definitions of strongly convex and smooth functions $f:\mathbb{R}^d \mapsto \mathbb{R}$.

\begin{definition}\label{def:strong_cvx}
$f$ is $\mu$-strongly convex if $\forall w,w'$, $f(w') \ge f(w) + \langle \nabla f(w), w' - w \rangle + \frac{\mu}{2}\|w' - w\|^2$.
\end{definition}

\begin{definition}\label{def:smooth}
$f$ is $L$-smooth if $\forall w,w'$, $\|\nabla f(w) - \nabla f(w')\| \le L\|w -w' \|$.
\end{definition}

\begin{definition}\label{def:lipschitz}
$f$ is $L_k$ Lipschitz for every coordinate $k \in [d]$ if, $|\partial_k f(w)| \leq L_k$, where $\partial_k f(w)$ denotes the $k$-th coordinate of $\nabla f(w)$.
\end{definition}

We have the following assumptions on the loss function.
\begin{assumption}[Strong convexity]\label{assumption:str_cvx}
    The loss per sample $f(w,.)$ is $\mu$-strongly convex with respect to $w$.
\end{assumption}
\begin{assumption}[Smoothness]\label{assumption:smooth}
    The loss per sample $f(w,.)$ is also $L$-smooth with respect to $w$.
\end{assumption}

\begin{assumption}[Lipschitz]\label{assumption:lipschitz}
    The loss per sample $f(w,.)$ is $L_k$-Lipschitz for every coordinate $k\in[d]$.  Define $\Hat{L} = \sqrt{\sum_{k=1}^d L_k^2}$.
\end{assumption}
We want to emphasize that the above assumptions are standard and have appeared in the previous literature. For example, the strong convexity and smoothness conditions are often required to obtain theoretical guarantees for clustering (see \cite{ghosh_efficient_2021,lu2016statistical}, which includes \texttt{IFCA} and the classical $k$-means which assume a quadratic objective. The coordinate-wise Lipschitz assumption is also not new and (equivalent assumptions) featured in previous works (see \cite{pmlr-v80-yin18a,pmlr-v97-yin19a}.
We are now ready to show the guarantees of several subroutines of \texttt{SR-FCA}. First, we show the probability of error after the \texttt{ONE\_SHOT} step. Throughout this section, we require Assumptions~~\ref{assumption:str_cvx}, \ref{assumption:smooth} and \ref{assumption:lipschitz} to hold.
\begin{lemma}[Error after \texttt{ONE\_SHOT}]\label{thm:init}
After running \texttt{ONE\_SHOT} with $\eta\leq \frac{1}{L}$ for $T$ iterations, for the threshold $\lambda \in(\epsilon_1,\epsilon_2) $ and some constant $b_2 > 0$, the probability of error is
\begin{align*}
\Pr[C_0 \neq C^\star] \leq p \equiv  md\hspace{2mm}\exp(-n \frac{b_2\Delta}{\Hat{L} \sqrt{d}}),
\end{align*}
provided $\frac{n^{2/3}\Delta^{4/3}}{D^{2/3}\Hat{L}^{2/3}} \lesssim d$, where $\Delta = \frac{\mu}{2}(\frac{\min\{\epsilon_2- \lambda,\lambda - \epsilon_1\}}{2} - (1 - \frac{\mu}{L})^{T/2}D)$ and $n = \min_{i \in[m]}n_i$. 
\end{lemma}
The proof of this Theorem is presented in Appendix~\ref{sec:init_proof}. We would like to emphasize that the probability of error is exponential in $n$, yielding a \emph{reasonable} good clustering after the \texttt{ONE\_SHOT} step. Note that the best probability of error is obtained when $\lambda = \frac{\epsilon_1 + \epsilon_2}{2}$.
\begin{remark}[Separation]
In order to obtain $p < 1$, we require $\Delta = \Omega (\frac{\log m}{n})$. Since $\Delta \leq \frac{\mu}{2}\frac{\epsilon_2 - \epsilon_1}{4}$, we require $(\epsilon_2 -\epsilon_1) \geq \mathcal{O}(\frac{\log m}{n})=\Tilde{\mathcal{O}}(\frac{1}{n})$.  Note that we require a condition only on the separation $\epsilon_2 - \epsilon_1$, instead of just $\epsilon_2$ or $\epsilon_1$ individually
\end{remark}
\begin{remark}[Improved separation compared with \texttt{IFCA}]
Let us now compare the separation with that of \texttt{IFCA}. Note that for \texttt{IFCA}, $\epsilon_1 = 0$, and the separation is $\Tilde{\mathcal{O}}(\max\{\frac{\alpha^{-2/5}}{n^{1/5}}, \frac{\alpha^{-1/3}}{n^{1/3} m^{1/6}} \})$, where $\alpha > 0$ is the initialization factor. In the regime where $\alpha = \mathcal{O}(1)$, \texttt{IFCA} requires a separation of $\Tilde{\mathcal{O}}(\frac{1}{n^{1/5}})$, which is much worse compared to \texttt{SR-FCA} which requires a separation of $\Tilde{\mathcal{O}}(\frac{1}{n})$.
\end{remark}
Although we obtain an exponentially decreasing probability of error, we would like to improve this dependence. \texttt{REFINE()} step does this job. The theorem below shows the improvement for a single step of \texttt{REFINE()}.
\begin{theorem}[ One step \texttt{REFINE()}]\label{thm:refine}
Let $\beta t = \Theta(c_{\min})$, and \texttt{REFINE()} is run with \texttt{TrimmedMeanGD}($\beta$). Provided,
\begin{align*}
\min\{\frac{n^{2/3}\Delta'^{4/3}}{D^{2/3}}, \frac{n^2\Delta'^2}{\Hat{L}^2\log(c_{\min})}\} \gtrsim  \,d,
\end{align*}
with $0 < \beta < \frac{1}{2}$, where $\Delta' = \Delta - \frac{\mu B}{2} > 0$ and $B =\sqrt{\frac{2\Hat{L}\epsilon_1}{\mu}}$. Then, 
for any constant $\gamma_1 \in (1,2)$ and $\gamma_2 \in (1,2-\frac{\mu B}{2\Delta})$,  such that after running 1 step of \texttt{REFINE()} with $\eta \leq \frac{1}{L}$, we have
\begin{align*}
   & \Pr[C_1 \neq C^\star] \leq \frac{m}{c_{\min}}\exp(- a_1 c_{\min}) +  \frac{m}{t}\exp(-a_2 m) + (1-\beta)m (\frac{p}{m})^{\gamma_1}   +  m (\frac{p}{m})^{\gamma_2}+ 8d\frac{m}{t} \exp(-a_3 n\frac{\Delta'}{2\Hat{L}}),
    \end{align*}
where $c_{\min}$ is the minimum size of the cluster. Further for some small constants $\rho_1 > 0,\rho_2 \in (0,1)$, we can select $\beta, \gamma_1$ and $\gamma_2$ such that for large $m,n$ and $\Delta'$, with $B << \frac{2\Delta'}{\mu}$, we have $\Pr[C_1 \neq C^\star] \leq \frac{\rho_1}{m^{1-\rho_2}} p$.
\end{theorem}
\begin{remark}[Misclustering error improvement]
Note that $\rho_2$ can be made arbitrarily close to $0$ by a proper choice of $\gamma_1$ and $\gamma_2$. So, one step of  \texttt{REFINE()} brings down the misclustering error by (almost) a factor of $1/m$, where $m$ is the number of worker machines.
\end{remark}
\begin{remark}[Condition on $B$]
Note that we require $B << \frac{2\Delta'}{\mu}$ for the above to hold. From the definition of $B$, when the intra-cluster separation $\epsilon_1$ is small, $B$ is small. So, setup like \texttt{IFCA}, where $\epsilon_1 =0$, this condition is automatically satisfied.
\end{remark}
We now run the \texttt{REFINE()} step for $R$ times. We have the following result.
\begin{theorem}[Multi-step \texttt{REFINE()}]
If we run $R$ steps of \texttt{REFINE()}, resampling $n_i$ points from $\cD_i$ and recompute $w_{i}$ as in \texttt{ONE\_SHOT} for every step of \texttt{REFINE()}, then the probability of error for \texttt{SR-FCA} with $R$ steps of \texttt{REFINE()} is 
\begin{equation}
    \Pr[C_R\neq C^\star] \leq \big(\frac{\rho_2 }{m^{(1-\rho_1)}}p\big)^R
\end{equation}
\begin{remark}[Resampling]
Note that although the theoretical convergence of Multi-step \texttt{REFINE()} requires resampling of data points in each iteration of \texttt{REFINE()}, we experimentally validate (see Section~\ref{sec:experiments}, that this is not required at all.
\end{remark}
\begin{remark}
In experiments (Section~\ref{sec:experiments}), we observe that it is often sufficient to run $1-2$ steps of \texttt{REFINE()}. Since each step of \texttt{REFINE()} reduces the probability of misclusteing by (almost) a factor of $1/m$, very few steps of \texttt{REFINE()} is often sufficient.
\end{remark}
\end{theorem}
\subsection{Convergence of cluster iterates:}
Apart from misclustering error, we also obtain an appropriate loss minimizer for each cluster, defined in Eq.~\eqref{eq:wstar}---which along  with our clustering model, is a reasonable good approximation for all the machines in that cluster.
\begin{theorem}[Cluster iterates]\label{thm:convergence}
Under the conditions described in Theorem~\ref{thm:refine}, after running \texttt{SR-FCA} for $(R+1)$ steps of \texttt{REFINE()}, we have  $C^{R+1} = C^\star$ and 
\begin{align*}
    && \norm{\omega_{c,T} - \omega_c^\star} \leq (1-\kappa^{-1})^{T/2}D + \Lambda + 2B, \hspace{1mm}\text{where,} \hspace{1mm}  \Lambda= \cO\bigl(\frac{\Hat{L}d}{1 - 2\beta}\bigl(\frac{\beta}{\sqrt{n}} + \frac{1}{\sqrt{n c_{\min}}}\bigr)\sqrt{\log(n m\Hat{L}D)}\bigr)
\end{align*}
 $\forall c\in \range(C^\star)$, with probability $1 - \big(\frac{\rho_2 }{m^{(1-\rho_1)}}p\big)^R - \frac{m}{c_{\min}}\frac{4d u''}{(1 + n c_{\min}\Hat{L}D)^d}$, for some constant $u'' >0$.
\end{theorem}
The proof of the above theorem is presented in Appendix~\ref{sec:convergence_proof}.
\begin{remark}[Convergence rate matches \texttt{IFCA}]
Note that the iterates converge exponentially fast to the true cluster parameter $\omega_c^\star$, which matches the convergence speed of \texttt{IFCA}.
\end{remark}
\begin{remark}[Comparison with \texttt{IFCA} in statistical error]
Let us now compare the error rate with that of \texttt{IFCA}. Note that for \texttt{IFCA}, $\epsilon_1 = 0$ and the statistical error rate of \texttt{IFCA} is $\Tilde{\mathcal{O}}(1/n)$ (see \cite[Theorem 2]{ghosh_efficient_2021}). Looking at Theorem~\ref{thm:convergence}, we see that under similar condition ($\epsilon_1 = 0$ and hence $B=0$), \texttt{SR-FCA} obtains an error rate of $\Tilde{\mathcal{O}}(1/\sqrt{n})$, which is weaker than \texttt{IFCA}. This can be thought of the price of initialization. In fact for \texttt{IFCA}, a \emph{good} initialization implies that only a very few machines will be mis-clustered, which was crucially required to obtain the $\Tilde{\mathcal{O}}(1/n)$ rate. But, for \texttt{SR-FCA}, we do not have such guarantees to begin with, and we necessarily take a union bound on \emph{all} machines, which results in a weaker statistical error.
\end{remark}
\begin{remark}[Potential improvement, matching statistical error of \texttt{IFCA}]
We use Theorem~\ref{thm:refine} to bound probability of error for first $R$ steps and then for the $(R+1)^{th}$ step we use analysis from ~\cite{pmlr-v80-yin18a} to optimize for cluster iterate convergence.
Note that the cluster estimates from \texttt{SR-FCA} are thus good approximations of $\omega_c^\star$. Our iterates are obtained via \texttt{TrimmedMeanGD()} which assumes $\beta$ fraction of nodes inside each cluster are corrupted. Instead, if we run any federated optimization algorithm which can accommodate low heterogeneity, for instance FedProx~\cite{fedprox}, inside each cluster $C_R$, then we can shave off the $\Lambda$ term from Theorem~\ref{thm:convergence}, to obtain convergence to a neighborhood of radius $2B$ of $\omega_c^\star$ for each cluster $c\in C^\star$.
\end{remark}

%% file: experiments.tex
\section{Experiments}
\label{sec:experiments}

\subsection{Setup}
We compare the empirical performance of \texttt{SR-FCA} against several baselines for various datasets.
Throughout the experiments, we emphasize that \texttt{SR-FCA} does not require the knowledge of the gap $\epsilon_1$ and $\epsilon_2$ or the minimum size of clusters $c_{\min}$. Our algorithm takes $\lambda$ and $t$ as hyperparameters,  where $\lambda$ is set by tuning and $t=2$. The trimming level, $\beta$ in the \texttt{TrimmedMeanGD} subroutine, as well as all the optimizers for each problem are also obtained by tuning.  Further, we recover the clusters via correlation clustering. Since correlation clustering is NP-Hard, we use a randomized approximation algorithm~\cite{Bansal02correlationclustering} for it. Note that in all our subroutines (\texttt{RECLUSTER}, \texttt{MERGE}, \texttt{ONE\_SHOT}), we remove any clusters which have $< t$ clients. We use $\ell_2$ norm as the distance metric for Synthetic case and for other cases, we use cross-cluster loss metric.

We compare \texttt{SR-FCA} against 3 baselines -- local, global and \texttt{IFCA}, for $3$ random seeds. The local baseline assumes that every client trains its own local model and the global baseline trains a single model via FedAvg~\cite{mcmahan2017federated} for all machine. We report the final test performance (loss or accuracy), by averaging over the clients their test performance of each client on its model, which for local is the local model, for global is the single global model and for \texttt{SR-FCA} and \texttt{IFCA} is the cluster model for the client.

We use 3 kinds of federated datasets -- synthetic, simulated and real. We generate the synthetic datasets on our own, while the simulated datasets are generated from standard datasets like MNIST~\cite{mnist} and CIFAR10~\cite{cifar10}. In these two cases, the heterogeneity and actual number of clusters and cluster identity is under our control, so for these cases we also check if \texttt{SR-FCA} is able to recover the cluster structure. For real federated datasets, which are obtained from leaf~\cite{leaf} database, we do not have this information, so we only compare final test performance.

\paragraph{Synthetic:}
To verify our theoretical results, we first test \texttt{SR-FCA} for mixture of linear regression. Here, we assume two clusters, each with a different $w^\star$ having dimension $d = 1000$. Each coordinate of $w^\star$ is generated iid from a $Bernoulli(0.5)$ distribution. Each coordinate of the  feature vector $x$ is sampled iid from  $\cN(0,1)$,  and we generate the target $y$ as
$y = \lin{x,w^\star} + \epsilon $, where the noise $\epsilon \overset{iid}{\sim}\cN(0,\sigma^2)$. We set $\sigma = 0.001$. 
We generate $m=100$ machines divided equally into the two clusters with $n=100$ datapoints per machine. We fit a linear model by minimizing the least squares loss for $280$ iterations and report the final test metric. For \texttt{SR-FCA}, we use $1$ refine step with $\ell_2$ norm as the distance metric, as described in Proposition~\ref{prop:lin_reg}.

\paragraph{Simulated Datasets--MNIST and CIFAR 10:}
We generate heterogeneous federated datasets from standard ML datasets, MNIST and CIFAR10 by splitting each dataset into $m$ disjoint sets of $n$ data points each, one per client and inject heterogeneity via pixel inversion and rotation. For MNIST, by inverting pixel value in MNIST, we create 2 clusters (referred to as inverted in Table~\ref{tab:experiments}) and by rotating the image by $90,180,270$ degrees we get 4 clusters. For CIFAR10, we create 2 clusters by rotating the images by $180$ degrees.
Applying rotations is a common practice in continual learning setup~\cite{rot_mnist} and also used in past FL literature~\cite{ghosh_efficient_2021}.
For MNIST and CIFAR10, we set $(m=100, n=600)$ and $(m=32, n=3125)$ respectively.

To emulate practical FL scenarios, we assume that only a fraction of the nodes participate in the learning procedure. For Rotated and Inverted MNIST, we assume that all the nodes participate, while for Rotated CIFAR10 50\% of the nodes participate.

We use a 2-layer fully connected feed-forward Neural Net (NN)  with $200$ hidden units for MNIST, ResNet9~\cite{resnet9} for CIFAR10 and a CNN with 2 convolution and fully connected layers for FEMNIST. For MNIST and CIFAR10, we train the models from scratch while for FEMNIST, we start from a model pre-trained by FedAvg. 
We train Rotated MNIST, Inverted MNIST and Rotated CIFAR10 for 250, 280 and 2400 iterations respectively with $2$ refine steps for \texttt{SR-FCA}.
\paragraph{Real Dataset:}
We use two real federated datasets from leaf~\cite{leaf} database -- FEMNIST and Shakespeare.  

We sample $m=50$ machines from each of these datasets. FEMNIST is a federated version of EMNIST~\cite{cohen2017emnist}, where each client has handwritten images from a single person.
Shakespeare dataset contains dialogues from Shakespeare's plays where every client has dialogues of a single character. The task for Shakespeare is next character prediction.

Since these are real  federated datasets, we do not know the correct number of clusters for them, therefore, we run \texttt{IFCA} for $3,4,5$ clusters respectively and report the average test accuracy.

For FEMNIST, we use a CNN with 2 convolution and fully connected layers and for Shakespeare, we use a 2-layer Stacked LSTM with an embedding layer. We run FEMNIST and Shakespeare for $1000$ and $2400$ iterations respectively and set number of refine steps to be $2$ for \texttt{SR-FCA}.

\subsection{Results} 
From Table~\ref{tab:experiments} we can see that across all different datasets,  \texttt{SR-FCA} outperforms all baselines. 

\paragraph{Comparison with Local and Global Baselines:}
Note that the local model has access to very little data, while the global model cannot handle the heterogeneity in different clusters. Therefore, both \texttt{IFCA} and \texttt{SR-FCA} outperform these baselines as they identify correct clusters with low heterogeneity inside each cluster. Further, the difference in test performance between these baselines and \texttt{SR-FCA} increases if we reduce the number of datapoints per client or increase the heterogeneity between clients. 
\begin{table*}
\parbox{\textwidth}{
\small
\caption{Test Performance on various datasets. For Synthetic, we report the test loss and for the rest, we report the test accuracy}
\label{tab:experiments}

\begin{center}
\begin{sc}
\begin{tabular}{ccccccc}
\toprule
Model & \makecell{Synthetic}& \makecell{MNIST\\ (inverted)}&\makecell{MNIST \\(rotated)}  & \makecell{CIFAR \\(rotated)} & \makecell{FEMNIST} & \makecell{Shakespeare}\\
\midrule
\texttt{SR-FCA}  &\textbf{4.0224}& \textbf{92.84}   & \textbf{91.83} & \textbf{88.7} &\textbf{84.93} & \textbf{47.68}\\
Local  &4.9752&   82.24& 85.82& 76.0& 75.54&32.72\\
Global & 4.1141& 88.44 &85.17 & 88.1 & 81.96 &46.99\\
\texttt{IFCA} & 4.0450&91.89&91.6&87.6&83.02&44.35\\
\bottomrule
\end{tabular}
\end{sc}
\end{center}
}
\hfill
\end{table*}

\paragraph{Recovering the correct clustering $\mathcal{C}^\star$:}
Apart from the real datasets, the correct clustering is known to us beforehand.
In these cases, \texttt{SR-FCA} recovers the correct clustering, if the distance metric and the threshold $\lambda$ is tuned properly. The task of tuning $\lambda$
is especially simplified for simulated datasets, as the heterogeneity for clients inside the cluster (i.e., $ \epsilon_1$) is 0. In this setup, we observe that \texttt{ONE\_SHOT} is near sufficient to recover $\mathcal{C}^\star$. For real datasets, this observation is not true and we require multiple \texttt{REFINE} steps.

\paragraph{Comparison with \texttt{IFCA}:}
On simulated and synthetic datasets, the number of clusters is already known, thus the performance of \texttt{SR-FCA} is similar to that of \texttt{IFCA}. However, for real datasets this is unknown, and \texttt{SR-FCA} clearly outperforms \texttt{IFCA}, especially for Shakespeare. 

We first emphasize that the accuracy we report here are somewhat different compared to \cite{ghosh_efficient_2021}. The main reason is that for fair comparison, we average the accuracy over multiple random seed, while in \cite{ghosh_efficient_2021}, the maximum accuracy over seeds is reported. 

As we run \texttt{IFCA} without initialization, we find that if clusters are empty early in training, then they remain empty throughout training. For Shakespeare, with $K=5$ clusters, we ended up with 2 empty clusters. \texttt{SR-FCA} does not suffer from these issues due to the appropriate initialization from \texttt{ONE\_SHOT}.
 
 Additionally, note that the number of clusters in a small sample of a real dataset is not fixed, therefore averaged over different seeds, IFCA performs poorly for each value of $K$. \texttt{SR-FCA} on the other hand, can compute both the clustering and cluster iterates without the knowledge of $K$, thereby beating IFCA.
 
An interesting example of the issues with \texttt{IFCA} is the Shakespeare dataset. \texttt{SR-FCA} recovers exactly $1$ cluster in this case, therefore it's accuracy is close to the global baseline. \texttt{IFCA}, on the other hand, tries to fit $K=3,4,5$ clusters to this dataset and therefore performs poorly.
\paragraph{Key Takeaways from experiments:}
We summarize the key takeaways from our experiments --
(1) \texttt{SR-FCA} outperforms all baselines across all datasets; (2) \texttt{SR-FCA} recovers the true clustering $C^\star$ for synthetic and simulated datasets; (3)\texttt{SR-FCA} outperforms \texttt{IFCA} especially for real datasets.

\paragraph{Acknowledgements}
This research is supported in part by NSF awards 2112665, 2217058,  and 2133484.

%% file: proof.tex
\newpage
\appendix
\onecolumn
\begin{center}
    \textbf{\LARGE{Appendix}}
\end{center}

\section{Algorithm Description}
\label{sec:complete_algo}
We provide complete descriptions for the subroutines in \texttt{REFINE} namely -- \texttt{TrimmedMeanGD}, \texttt{RECLUSTER} and \texttt{MERGE}.

\begin{algorithm}[htb]
  \caption{\texttt{TrimmedMeanGD()}}
  \label{alg:trmean}
\begin{algorithmic}
    \STATE {\bfseries Input:} $0\leq \beta <\frac{1}{2}$, Clustering $\mathcal{C}_r$
    \STATE {\bfseries Output:} Cluster iterates $\{\omega_{c}\}_{c\in \range(\mathcal{C}_r)}$
    \FOR{ all clusters $c \in  \range(\mathcal{C}_r)$ in parallel}
    \STATE $w_{c,0} \gets w_0$
    \FOR{ $t=0$ to $T-1$}
        \STATE $g(w_{c,t}) \gets \mathrm{TrMean}_{\beta}(\{\nabla f_i(w_{c,t}), \mathcal{C}_r(i) = c\})$
        \STATE $w_{c,t+1} \gets proj_{\cW}\{w_{c,t} - \eta  g_t\}$
    \ENDFOR
    \STATE \textbf{Return} $\{\omega_{c,T}\}_{c\in \range(\mathcal{C}_r)}$
    \ENDFOR
\end{algorithmic}
\end{algorithm}

\begin{algorithm}[tbh]
  \caption{\texttt{RECLUSTER()}}
  \label{alg:recluster}
\begin{algorithmic}
    \STATE {\bfseries Input:} Cluster iterates $\{\omega_{c}\}_{c\in \range(\mathcal{C}_r)}$, Node iterates $\{w_{i}\}_{i=1}^m$, Clustering $\mathcal{C}_r$
    \STATE {\bfseries Output:} Improved Clustering $\mathcal{C}_{r}'$
    \FOR{all nodes $i\in[m]$}
        \STATE $\mathcal{C}'_{r}(i) \gets \argmin_{c\in \range(\mathcal{C}_r)} \mathsf{dist}(w_{i}, \omega_{c})$
    \ENDFOR
    \STATE \textbf{return} Clustering $\mathcal{C}_r'$.
\end{algorithmic}
\end{algorithm}
\vspace{-2mm}

\textbf{Subroutine TrimmedMeanGD():} 
The full algorithm for \texttt{TrimmendMeanGD} is provided in Algorithm~\ref{alg:trmean}.

\begin{definition}[$\mathrm{TrMean}_{\beta}$]\label{def:trmean}
For $\beta \in [0,\frac{1}{2})$, and a set of vectors $x^j \in \R^d, j\in [J]$, their trimmed mean $g = \mathrm{TrMean}_\beta(\{x^1, x^2,\ldots, x^J\})$ is a vector $g\in \R^d$, with each coordinate $g_k = \frac{1}{(1-2\beta)J}\sum_{x \in U_k}x$, for each $k \in [d]$, where $U_k$ is a subset of $\{x_k^1,x_k^2,\ldots, x_k^J\}$ obtained by removing the smallest and largest $\beta$ fraction of its elements. 
\end{definition}

Note that $\mathrm{TrMean}_{\beta}$ has been used to handle Byzantine nodes achieving optimal statistical rates~\cite{pmlr-v80-yin18a}. $\mathrm{TrMean}_{\beta}$ can handle atmost $<\beta$ fraction of the nodes being byzantine, therefore, we need \texttt{ONE\_SHOT} to return clusters where each has $< \beta$ fraction of impurities for our theoretical results.

Note that in our experiments, we use $>1$ local steps and take $\texttt{TrMean}_{\beta}$ when averaging local models.

\textbf{Subroutine RECLUSTER():} The full algorithm is provided in Algorithm~\ref{alg:recluster}. Each client is sent to the cluster which is closest to it, in terms of $\mathsf{dist}(w_i, \omega_c)$.

\textbf{Subroutine MERGE():} 
The full algorithm is provided in Algorithm~\ref{alg:merge}. Similar to \texttt{ONE\_SHOT}, we create a graph $G$ but instead with vertex set being the clusters in $\mathcal{C}_r'$. Then, we add edges between clusters based on the threshold and find all the clusters in the resultant graph $G$ by correlation clustering. Then, each of these clusters in $G$ correspond to a set of clusters in $\mathcal{C}_r'$, so we merge them into a single cluster to obtain the final clustering $\mathcal{C}_{r+1}$.

\begin{algorithm}[t!]
  \caption{\texttt{MERGE()}}
  \label{alg:merge}
\begin{algorithmic}
    \STATE {\bfseries Input:} Cluster iterates $\{\omega_{c}\}_{c\in \range(\mathcal{C}_r)}$ , Clustering $\mathcal{C}_r'$, Threshold $\lambda$, Size parameter $t$
    \STATE {\bfseries Output:} Merged Clustering $\mathcal{C}_{r+1}$, Cluster iterates $\{\omega_{c}\}_{c\in \range(\mathcal{C}_{r+1})}$
        \STATE $G \gets$ Graph with vertex set $\range(\mathcal{C}_r')$  and no edges
        \FOR{all pairs of clusters $c,c'\in \range(\mathcal{C}_r'), c\neq c' $}
        \STATE Add edge $(c,c')$ to the graph $G$ if $\mathsf{dist}(w_{c},w_{c'}) \leq \lambda$
        \ENDFOR
        \STATE $\mathcal{C}_{temp} \gets$ Obtain clusters from graph $G$ with size $\geq t$ by correlation clustering of \cite{Bansal02correlationclustering}.
        \STATE For each cluster in $\mathcal{C}_{temp}$, merge the nodes of its component clusters to get $\mathcal{C}_{r+1}$
    \FOR{$c \in \range(\mathcal{C}_{temp})$}
        \STATE $G_c \gets \{ c' \in \range(\mathcal{C}_r'): c' \text{ has been merged into } c  \}$
        \STATE $\omega_{c} \gets \frac{1}{\abs{G_c}}\sum_{c' \in G_c} \omega_{c'}$
    \ENDFOR
    \STATE \textbf{return}  $\mathcal{C}_{r+1}, \{\omega_{c}\}_{c\in \range(\mathcal{C}_{r+1})}$.
\end{algorithmic}
\end{algorithm}

\section{Proof of Proposition~\ref{prop:lin_reg}}
According to the proposition, for two users $i$ and $j$, the data is generated by first sampling each coordinate of $x\in\R^d$ from $\cN(0,1)$ iid and then computing $y$ as --
\begin{align*}
    y_i = \lin{x , w_i^\star} + \epsilon_i
\end{align*}
where $\epsilon_i \overset{iid}{\sim} \cN(0,\sigma^2) $. Then, the distribution of $y_i|x$ is $\cN(\lin{x,w_i^\star},\sigma^2)$.
Therefore, the $KL$ divergence between $y_i|x$ and $y_j|x$ is given by
\begin{align*}
    KL(p(y_i|x)|| p(y_j|x)) = \frac{\lin{w_i^\star - w_j^\star, x}^2}{2\sigma^2}
\end{align*}
Therefore, if we take expectation wrt $x$, we have
\begin{align*}
    \E_x[KL(p(y_i|x)|| p(y_j|x))] = \frac{d\norm{w_i^\star - w_j^\star}^2}{2\sigma^2}
\end{align*}

\section{Proof of Lemma~\ref{thm:init}}
\label{sec:init_proof}

Note that throughout the proof, we treat $w_{i,T}$ as the output of $i^{th}$ node after training for $T$, instead of $w_i$ and $w_{c,T}$ as the output of the $c^{th}$ cluster after \texttt{TrimmedMeanGD} for $T$ iterations.

In \texttt{ONE\_SHOT()}, $\mathcal{C}_0 = \mathcal{C}^\star$, if all the edges formed in the graph are correct. This means that if $i,j$ are in the same cluster in $\mathcal{C}^\star$, then $\norm{w_{i,T} - w_{j,T}} \leq \lambda$ and if $i,j$ are in different clusters, $\norm{w_{i,T} - w_{j,T}} > \lambda$. 

Note that, 
\begin{align*}
   w_{i,T} - w_{j,T} = (w_{i}^\star - w_j^\star)+(w_{i,T} - w_{i}^\star)  -(w_{j,T} - w_j^\star) 
\end{align*}
Now, if we apply triangle inequality, we obtain
\begin{align*}
    &\mathsf{dist}(w_{i,T},w_{j,T}) \geq\mathsf{dist}(w_{i}^\star,w_j^{\star}) - \Xi_{i,j},\quad\mathsf{dist}(w_{i,T}, w_{j,T}) \leq \mathsf{dist}(w_{i}^\star, w_j^{\star}) + \Xi_{i,j}
\end{align*}
where $\Xi_{i,j} =  \sum_{k=i,j}\mathsf{dist}(w_{k,T}, w_{k}^\star)$. This decomposition forms the key motivation for our algorithm.

Therefore, if $i,j$ are in the same cluster, then a sufficient condition for  edge $(i,j)$ to be incorrect is
\begin{align*}
    \lambda &\leq \mathsf{dist}(w_i^\star, w_j^\star) + \Xi_{i,j}\\
    \Xi_{i,j} &\geq \lambda - \epsilon_1
\end{align*}
Similarly, if $i,j$ are in different clusters, then a sufficient condition for edge $(i,j)$ to be incorrect is
\begin{align*}
    \lambda &\geq \mathsf{dist}(w_i^\star, w_j^\star) - \Xi_{i,j}\\
    \Xi_{i,j} &\geq \epsilon_2 - \lambda
\end{align*}

Therefore, we can set $\Delta_{\lambda} = \min\{\epsilon_2 - \lambda, \lambda - \epsilon_1\}$, and then a sufficient condition for any edge to be  incorrect is $\max_{i,j} \Xi_{i,j} \geq \Delta_{\lambda}$. 

Thus, 
\begin{align}
    \Pr[\mathcal{C}^\star \neq \mathcal{C}_0] \leq& \Pr[\text{at least 1 edge is incorrect}]\\
    \leq& \Pr[\max_{i,j}\, \Xi_{i,j} \geq \Delta_{\lambda}]\\
    \leq& \Pr[\max_{i,j} \sum_{k=i,j}\norm{w_{k,T} - w_{k}^\star} \geq \Delta_{\lambda}]\\
    \leq& \Pr[\max_{i,j} \max_{k=i,j}(\norm{w_{k,T} - w_{k}^\star} \geq \frac{\Delta_{\lambda}}{2}]\\
    \leq& \Pr[\max_{i\in[m]} \,\norm{w_{i,T} - w_{i}^\star} \geq \frac{\Delta_{\lambda}}{2}]\label{eq:init_proof_int}
\end{align}
The second and third inequalities are obtained by expanding the terms. The fourth inequality is obtained by $\Pr[a + b \geq c] \leq \Pr[\max\{a,b\}\geq c/2]$. 
For the fifth inequality, we merge  $\max_{i,j}\max_{k=i,j}$ into $\max_{i\in[m]}$. 
As we can see in Equation~\eqref{eq:init_proof_int}, we need to bound $\norm{w_{i,T} - w_i^\star}$ for each node $i$. The subsequent Lemma allow us to bound this quantities.

\begin{lemma}[Convergence of $w_{i,T}$]\label{lem:conv_node}
    Let  $\frac{n^{2/3}\Delta^{4/3}}{D^{2/3}\Hat{L}^{2/3}} \lesssim b_1 d$, for some constant $b_1 >0$. Then, after running \texttt{ONE\_SHOT()} with $\eta\leq \frac{1}{L}$, for some constant $b_2 > 0$, under Assumptions~\ref{assumption:str_cvx},\ref{assumption:smooth} and \ref{assumption:lipschitz}, we have
    \begin{align*}
    \Pr[\norm{w_{i,T} - w_i^\star} \geq \frac{\epsilon_2 - \epsilon_1}{4}] \leq  d\hspace{2mm}\exp(-n \frac{b_2\Delta}{\Hat{L} \sqrt{d}}),
    \end{align*}
    where $\Delta = \frac{\mu}{2}(\frac{\Delta_{\lambda}}{2} - (1 - \frac{\mu}{L})^{T/2}D)$ and $n = \min_{i \in[m]}n_i$. 
\end{lemma}
This lemma follows from ~\citep{pmlr-v80-yin18a}. The complete proof of this Lemma is present in Section~\ref{sec:conv_node_proof}.

Now, we can apply Lemma~\ref{lem:conv_node} in Eq~\eqref{eq:init_proof_int}.
\begin{align*}
    \Pr[\mathcal{C}_0\neq \mathcal{C}^\star] \leq& \Pr[\max_{i \in [m]}\norm{w_{i,T} - w_i^\star} \geq \frac{\Delta_{\lambda}}{2}]\\
    \leq& m \max_{i \in [m]}\Pr[\norm{w_{i,T} - w_i^\star} \geq \frac{\Delta_{\lambda}}{2}]\\
    \leq& m d\hspace{2mm}\exp(-n \frac{b_2\Delta}{\Hat{L} \sqrt{d}})
\end{align*}
For the second inequality, we  use $\Pr[\max_{i \in[m]} a_i \geq c] \leq \sum_{i \in[m]} \Pr[a_i \geq c] \leq m \max_{i\in[m]}\Pr[a_i \geq c]$, which follows from union bound.

Note that for $p<1$, we need the separation to be order of $\Theta(\sqrt{\frac{\log m}{n}})$.

\subsection{Proof of Lemma~\ref{lem:conv_node}}
\label{sec:conv_node_proof}
We utilize results from ~\cite{pmlr-v80-yin18a}, which hold for  \texttt{TrimmedMeanGD} to analyze convergence for a single node as they yield stronger guarantees under the given assumptions.

\begin{lemma}[Convergence of $w_{i,T}$]
\label{lem:conv_node_int}
    If Assumptions~\ref{assumption:str_cvx},\ref{assumption:smooth} and ~\ref{assumption:lipschitz}~ hold, and $\eta \leq \frac{1}{L}$, then
        \begin{equation}
            \norm{w_{i,T} - w_i^\star} \leq (1-\kappa^{-1})^{T/2}D + \frac{2}{\mu}\Lambda_i \quad \forall i \in [m]  
        \end{equation}
        where  $\kappa = \frac{L}{\mu}$ and $\Lambda_i$ is a positive random variable with
        \begin{align}
            &\Pr[\Lambda_i \geq \sqrt{2d}r + 2\sqrt{2}\delta \Hat{L}]\leq 2d(1 + \frac{D}{\delta})^d\exp(- n \min\{\frac{r}{2\Hat{L}},\frac{r^2}{2\Hat{L}^2}\})
        \end{align}
        for some $r,\delta>0$.
    \end{lemma}
We provide the proof of this lemma in Appendix~\ref{sec:conv_node_int_proof}.

Using the above Lemma, we can bound the probability $\Pr[\norm{w_{i,T} - w_i^\star}\geq \frac{\Delta_{\lambda}}{2}]$
\begin{align*}
    \Pr[\norm{w_{i,T}- w_{i}^\star} \geq \frac{\Delta_{\lambda}}{2}]
    \leq& \Pr[ 2(1 - \kappa^{-1})^{T/2}D + \frac{2}{\mu}\Lambda_i +  \geq \frac{\Delta_{\lambda}}{2}]\\
    \leq& \Pr[ \Lambda_i \geq \Delta],\quad  \text{where } \Delta = \frac{\mu}{2}(\frac{\Delta_{\lambda}}{2} - (1 - \kappa^{-1})^{T/2}D)\\
    \leq &\Pr[\sqrt{2d}r + 2\sqrt{2} \delta \Hat{L}\geq \Delta]\\
        \leq & d\,  \exp(-n b_2\frac{\Delta}{\Hat{L} \sqrt{d}})
\end{align*}
for some constants $b_1, b_2, b_3, b_4 >0$, where we set $r = b_3\Hat{L} \max\{\frac{\Delta}{\Hat{L}\sqrt{d}},\sqrt{\frac{\Delta}{\Hat{L}\sqrt{d}}}\}$ and $\delta = b_4 \frac{\Delta}{\Hat{L}}$, and for $b_1 d \leq \frac{n^{2/3}\Delta^{4/3}}{D^{2/3}\Hat{L}^{4/3}}$, such that $\sqrt{2d}r + 2\sqrt{2} \delta \Hat{L}\geq \Delta$ and $n\min\{\frac{r}{2\Hat{L}},\frac{r^2}{2\Hat{L}^2}\} > \frac{Dd}{\delta}$ in Lemma~\ref{lem:conv_node_int}.

\section{Proof of Theorem~\ref{thm:refine}}
\label{sec:refine_proof}
\subsection{Preliminaries}
\label{sec:preliminaries}

First, we define certain random variables and their respective probabilities which we will use throughout this proof. Since the edge based analysis and corresponding clique identification involves a lot of dependent events, we try to decompose the absence/presence of edge into a combination of independent events.

Define, 
\begin{align}
    X_{ij} = \begin{cases}
      1 & \text{If the edge $(i,j)$ in $\mathcal{C}_0$ is incorrect in $\mathcal{C}^\star$} \\
      0 & \text{Otherwise} \\
\end{cases}
\end{align}
An edge $(i,j)$ in $\mathcal{C}_0$ is incorrect in $\mathcal{C}^\star$ if either it is present in $\mathcal{C}^\star$ and absent in $\mathcal{C}_0$ or vice versa.
We  analyze the probability of this event for the case when $\mathcal{C}^\star$ contains the edge $(i,j)$. The case when $\mathcal{C}^\star$ doesn't contain edge $(i,j)$ and it is present in $\mathcal{C}_0$ has exaclty same probability.  When $\norm{w_i^\star - w_j^\star } \leq \epsilon_1$, then edge is present is $\mathcal{C}^\star$. If it is absent in $\mathcal{C}_0$, then
\begin{align*}
    \Pr[X_{ij} = 1] \leq& \Pr[\Xi_{i,j} \geq \Delta_{\lambda}]\\
    \leq& \Pr[\Lambda_i + \Lambda_j \geq 2\Delta]
\end{align*}
The analysis is similar to the proof of \texttt{ONE\_SHOT()} in Appendix~\ref{sec:init_proof}.

Note that the random variables  $\{X_{ij}\}$ are not independent. We now define independent random variables $X_i$ such that
\begin{align}
    X_{i} = \begin{cases}
      1 & \text{If $\Lambda_i \geq \Delta$} \\
      0 & \text{Otherwise}
\end{cases}
\end{align}
Thus, we can see that $X_{ij} \leq X_i + X_j$. Additionally,
\begin{align}
    \Pr[X_i = 1] \leq \Pr[\Lambda_i \geq \Delta] \leq \frac{p}{m}
\end{align}
This follows from analysis of \texttt{ONE\_SHOT()} in Appendix~\ref{sec:init_proof}.

We can further generalize this notion to the random variables defined as $Y_{i,\gamma}$.
\begin{align}
    Y_{i,\gamma} = \begin{cases}
      1 & \text{If $\Lambda_i \geq \gamma\Delta, \gamma \in(0,2)$} \\
      0 & \text{Otherwise}
\end{cases}
\end{align}
Then, 
\begin{align*}
    \Pr[Y_{i,\gamma} = 1] \leq& \Pr[\Lambda_i\geq \gamma\Delta]
    \leq d \,\exp(-n b_2\frac{\gamma\Delta}{\hat{L} \sqrt{d}})= (\frac{p}{m})^{\gamma}
\end{align*}
Note that the set of random variables $\{Y_{i,\gamma}\}_{i=1}^m$ are mutually independent random variables.

Further, we define the $\omega_c^\star$ for every cluster $c \in \range(\mathcal{C}_0)$. Let $c' \in \mathcal{C}^\star$ be the cluster label of node $c$. If $G_{c} = \{i : i\in[m], \mathcal{C}^\star(i) = c'\}$, which is the set of nodes in $c$ which were from $c'$ in the original clustering, then
 we can define $\omega_c^\star$ and $F_c(w)$ as 
\begin{align}
    \omega_c^\star &= \argmin_{w \in \cW}\E[\frac{1}{\abs{G_{c'}}}\sum_{i \in G_{c'}}f_i(w)]\\
     &= \argmin_{w \in \cW}\frac{1}{\abs{G_{c'}}}\sum_{i \in G_{c'}} F_i(w)= \argmin_{w \in \cW} F_c(w)
\end{align}

We use this definition of $\omega_c^\star$ in the Appendix~\ref{sec:merge} and \ref{sec:recluster}.

\subsection{Analysis of \texttt{REFINE()}}
\label{sec:refine_analysis}
Our goal is to compute total probability of error for \texttt{REFINE()} to fail. If we define this error as $\mathcal{C}_1 \neq \mathcal{C}^\star$, then we can define the main sources of error for this event.

\begin{enumerate}
    \item \textbf{$\exists c \in \range(\mathcal{C}^\star)$ such that no cluster in $\mathcal{C}_0$ has cluster label $c$} : If the a cluster $c \in \range(\mathcal{C}^\star)$ is absent in $\mathcal{C}_0$, then subsequent steps of \texttt{REFINE()} will never be able to recover it, as they only involve node reclustering and merging existing clusters.
    The lemma presented below gives an upper bound on the probability of this event.  
    \begin{lemma}\label{lem:cluster_recover}
    Under the conditions of Theorem~\ref{thm:init} and if $t = \Theta(c_{\min})$, then there exists constant $a_1 >0$ such that 
    \begin{align*}
        \Pr[\exists c \in \range(\mathcal{C}^\star) \text{ such that no cluster in $\mathcal{C}_0$ has cluster label $c$}] \leq \frac{m}{c_{\min}}\exp(- a_1 c_{\min})
    \end{align*}
    \end{lemma}
    The proof of this Lemma is presented in Appendix~\ref{sec:cluster_recover_proof}
    \item \textbf{Each cluster $c \in \range(C)_0$ should have $<\alpha$ fraction of impurities for some $ \frac{1}{2}> \beta > \alpha$}: If some cluster has more than $\alpha$-fraction of impure nodes, then we cannot expect convergence guarantees for \texttt{TrimmedMeanGD}$_\beta$. 
    
    The below lemma bounds the probability of this error as
    \begin{lemma}\label{lem:alpha_fraction}.
    For some constants $0 <\alpha <\beta <\frac{1}{2}, a_2\geq 0, \gamma_1 \in (1,2)$ and $\alpha t = \Theta(m)$, under the conditions in Theorem~\ref{thm:init}, we have
    \begin{align*}
        \Pr[\exists c \in \range(\mathcal{C}_0)  \text{ which has $> \alpha$ fraction of impurities }]\leq  \frac{m}{t}\exp(-a_2 m) + (1-\alpha)m(\frac{p}{m})^{\gamma_1}
    \end{align*}
    \end{lemma}
    The proof of this Lemma is presented in Appendix~\ref{sec:alpha_fraction_proof}.
    \item \textbf{\texttt{MERGE()} error:} We define this as the error for the \texttt{MERGE()} to fail. Even though \texttt{MERGE()} operates after \texttt{RECLUSTER()}, \texttt{RECLUSTER()} does not change the cluster iterates. The goal of \texttt{MERGE()} is to ensure that all clusters in  $\mathcal{C}_0$ with the same cluster labels  are merged. Therefore, we define \texttt{MERGE()} error as the event when either two clusters with same cluster label are not merged or two clusters with different cluster labels are merged. The below lemma bounds this probability.
    \begin{lemma}\label{lem:merge_error}
    If $\min\{\frac{n^{2/3}\Delta^{4/3}}{D^{2/3}\Hat{L}^{2/3}},\frac{n^2\Delta'^2}{\Hat{L}^2\log(c_{\min})}\} \geq u_1 d$ for some constants $u_1>0$, then for some constant $a_3'>0$, where $\Delta' = \Delta - \frac{\mu B}{2} >0$, where $B =\sqrt{\frac{2\Hat{L}\epsilon_1}{\mu}}$, we have
    \begin{align*}
        \Pr[\texttt{MERGE()} \text{ Error}] \leq  \frac{4dm}{t} \exp(-a_3' n\frac{\Delta'}{2\Hat{L}})
    \end{align*}
    \end{lemma}
    The proof of this Lemma is presented in Appendix~\ref{sec:merge}.
    \item \textbf{\texttt{RECLUSTER()} error:} This event is defined as a node going to the wrong cluster after both \texttt{MERGE()}
    and \texttt{REFINE()} operations. After \texttt{MERGE()}, each cluster in $\mathcal{C}_0$ corresponds to a single cluster in $\mathcal{C}_1$. Therefore, we incur an error due to the \texttt{RECLUSTER()} operation if any node $i$ does not go to the cluster $c\in \mathcal{C}_1$ which has cluster label $\mathcal{C}^\star(i)$. The below lemma provides an upper bound on the probability of this error.
    \begin{lemma}\label{lem:recluster}
    If $\min\{\frac{n^{2/3}\Delta^{4/3}}{D^{2/3}\Hat{L}^{2/3}},\frac{n^2\Delta'^2}{\Hat{L}^2\log(c_{\min})}\} \geq u_2 d$ for some constants $u_2>0$, then for some constants $a_3'' >0$ and $\gamma_2\in (1,2-\frac{\mu B}{2\Delta})$, we have 
    \begin{equation}
        \Pr[\texttt{RECLUSTER()} error]     \leq 4d\frac{m}{t} \exp(-a_3'' n \frac{\Delta'}{2\Hat{L}}) +  m (\frac{p}{m})^{\gamma_2}
    \end{equation}
    \end{lemma}
    The proof of this Lemma is presented in Appendix~\ref{sec:recluster}.
\end{enumerate}
    
The total probability of error after for a single step of \texttt{REFINE()} is the sum of probability of errors for these 4 events by the union bound. Therefore,
\small
\begin{align}
\Pr[\mathcal{C}_1 \neq \mathcal{C}^\star] \leq \frac{m}{c_{\min}}\exp(- a_1 c_{\min}) +  \frac{m}{t}\exp(-a_2 m) + (1-\beta)m (\frac{p}{m})^{\gamma_1}    + 8d\frac{m}{t} \exp(-a_3 n\frac{\Delta'}{2\Hat{L}})+  m (\frac{p}{m})^{\gamma_2}
\end{align}
\normalsize

where we set $a_3 = \min\{a_3', a_3''\}$ . 

For some small constants $\rho_1>0, \rho_2 \in (0,1)$, we can choose $\gamma_1 \in(1,2), \beta \in(0,\frac{1}{2})$ and $\gamma_2 \in (1 , 2-\frac{\mu B}{2\Delta})$ such that $(1-\beta)(\frac{p}{m})^{\gamma_1 - 1} + (\frac{p}{m})^{\gamma_2 -1} \leq \frac{\rho_1}{2 m^{1-\rho_2}}$ and for large enough $m,\Delta'$ and $n$, $\frac{m}{c_{\min}}\exp(- a_1 c_{\min}) + \frac{m}{t}\exp(-a_2 m)  + 8d\frac{m}{t} \exp(-a_3 n\frac{\Delta'}{2\Hat{L}}) \leq \frac{\rho_1}{2m^{1-\rho_2}}p$. This happens because we have terms of $\exp(-m), \exp(-c_{\min})$ and $\exp(-n \Delta')$, which decrease much faster than $\frac{p}{m}$ which has terms of $\cO(m\exp(-n\Delta))$, where $\Delta$ and $\Delta'$ are of the same order. 
Therefore, the total probability of error can be bounded by

\begin{align}
    \Pr[\mathcal{C}_1 \neq \mathcal{C}^\star] \leq& \frac{\rho_1 }{m^{1-\rho_2}} p
\end{align}

\subsection{Proof of Lemma~\ref{lem:cluster_recover}}\label{sec:cluster_recover_proof}
\begin{align}
    &\Pr[\exists c \in \range(\mathcal{C}^\star) \text{ such that no cluster in $\mathcal{C}_0$ has cluster label $c$}] \leq \sum_{c \in \mathcal{C}^\star}\Pr[\text{No cluster in } \mathcal{C}_0 \text{ has cluster label $c$} ]\label{eq:cluster_recover_int}
\end{align}
Here, we use union bound over the clusters for the second inequality. 
Now, we analyze  the probability that no cluster in $\range(\mathcal{C}_0)$ has cluster label $c$ for some $c\in \range(\mathcal{C}^\star)$. Consider a cluster in $\range(\mathcal{C}_0)$. This cluster has cluster label $c$ if a majority of its nodes are from cluster $c \in \range(\mathcal{C}^\star)$. Since the size of each cluster in $\range(\mathcal{C}_0)$ is atleast $t$ and there are $C$ clusters in $\range(\mathcal{C}^\star)$, if all clusters in $\range(\mathcal{C}_0)$ have $\leq \frac{t}{C}$ nodes from cluster $c$, then no cluster will have cluster label $c$.

Assume that the clique formed by nodes from cluster $c$ has $r$ nodes. Then, every node $i$ in cluster $c$, must have $S_c - r$ edges absent, which correspond to the edges between a node of the clique and those outside it.
Thus, we obtain, 
\begin{align*}
   \Pr[\text{No cluster in } \mathcal{C}_0 \text{ has cluster label $c$} ] \leq& \Pr[\underset{\mathcal{C}^\star(i) = c}{\cap}\{\sum_{j\neq i, \mathcal{C}^\star(i) =c}X_{ij} > S_c - \frac{t}{C}\}]\\
    \leq& \Pr[\underset{\mathcal{C}^\star(i) =\mathcal{C}^\star(j) =c}{\sum\sum}X_{ij} > S_c(S_c - \frac{t}{C})]\\
   \leq& \Pr[\underset{\mathcal{C}^\star(i) =\mathcal{C}^\star(j) =c}{\sum\sum}(X_{i} + X_j) > S_c(S_c - \frac{t}{C})]\\
    \leq& \Pr[\frac{1}{S_c}\underset{\mathcal{C}^\star(i) =c}{\sum}X_{i} > 1 - \frac{t}{C S_c})]\\
    \leq& \exp(- \biggl(1 - \frac{t}{C S_c} - \frac{p}{m}\biggr)^2 S_c)\\
    \leq& \exp(- a_1 c_{\min})
\end{align*}
In the first step, we require each node $i$ to have $S_c-\frac{t}{C}$ wrong edges. For the second inequality, we remove the intersection and thus, the total number of incorrect edges has to be $S_c(S_c-\frac{t}{C})$, since each node has $S_c - \frac{t}{C}$ incorrect edges.
For the third inequality, we use $X_{ij}\leq X_i + X_j$ and collect the terms of $X_i$ for the fourth inequality. In the fifth inequality, we obtain a condition on the sum of independent Bernoulli random variables each with mean $\frac{p}{m}$. Therefore, we can apply Chernoff bound for their sum to obtain the fifth inequality.

A necessary condition for us is $1 - \frac{t}{CS_c} - \frac{p}{m} > 0$ which translates to $t < CS_c(1 - \frac{p}{m})$. If we select $t \leq c_{min} - 1$, this inequality is always satisfied. Note that we want the term $\biggl(1 - \frac{t}{CS_c} - \frac{p}{m}\biggr)^2  > a_1$, for some positive constant $a_1$. If we choose $t = \Theta(m)$, which is possible if $t = \Theta(c_{\min})$ as we assume $c_{\min} = \Theta(m)$, then this is satisfied. We use the lower bound $a_1$ and $S_c \geq c_{\min}$ to obtain the final inequality.
Plugging this in Eq~\eqref{eq:cluster_recover_int}, we obtain our result.

\subsection{Proof of Lemma~\ref{lem:alpha_fraction}}
\label{sec:alpha_fraction_proof}

\begin{align}
    &\Pr[\exists c \in \range(\mathcal{C}_0)  \text{ which has $\geq \alpha$ fraction of impurities}]
    \quad \leq \sum_{c \in \range(\mathcal{C}_0)}\Pr[\text{cluster $c$ has $\geq \alpha$ fraction of wrong nodes}]\label{eq:alpha_proof_int}
\end{align}
We use a simple union bound on clusters in $\mathcal{C}_0$ for the above inequality.
 Let the set of nodes in the cluster $c$ which are from same cluster of $\mathcal{C}^\star$ as the cluster label of $c$, i.e.,  which are not impurities, be $R_c$. Then let $Q_c = \abs{R_c}$. Let $Q'_c$ denote the number of impurities in cluster $c$. 
\begin{align*}
    \Pr[\text{cluster $c$ has $\geq \alpha$ fraction of wrong nodes}]\leq & \Pr[Q'_c \geq \frac{\alpha}{1 - \alpha}Q_c ]\\
    & \Pr[Q'_c \geq \alpha t ]
\end{align*}
We use the fact that $Q_c + Q'_c \geq t$, which is the minimum size of any cluster, for the second inequality.

Now, we analyze the probability of a single node to be incorrect.
A node is an impurity in cluster $c$ if it has an edge to each of nodes in $R_c$.
\begin{align}
    \Pr[\text{Node $i$ is an impurity in  cluster c}] \leq& \Pr[\min_{j \in R_c} \norm{w_{i,T} - w_{j,T}}\leq \lambda]\label{eq:alpha_fraction_int}\\
    \leq& \Pr[\min_{j \in R_c} (\norm{w_i^\star  - w_j^\star} - \Xi_{i,j}) \leq \lambda]\\
        \leq& \Pr[\Lambda_i  + \max_{j \in R_c} \Lambda_j\geq 2\Delta]
\end{align}
Now, if $\max_{j \in R_c}\Lambda_j\leq \gamma_1\Delta$, for $\gamma_1 \in (1,2)$, then we need $\Lambda_i\geq (2-\gamma_1)\Delta$ for error.

Using the definition of random variables in Appendix~\ref{sec:preliminaries}
\begin{align*}
    \Pr[Q'_c \geq \alpha t] \leq& \Pr[Q'_c \geq \alpha t | \max_{j \in R_c}\Lambda_j\leq \gamma_1\Delta] + \Pr[\max_{j \in R_c}\Lambda_j\geq \gamma_1\Delta]\\
     \leq& \Pr[\sum_{i=1}^m Y_{i,2-\gamma_1} \geq \alpha t] + \Pr[\max_{j \in R_c}\Lambda_j\geq \gamma_1\Delta]
\end{align*}
For the first inequality, we use union bound over the value of $\max_{j \in R_c}\Lambda_j$ and for the second inequality, we need atleast $\alpha t$ impurities, so atleast $\alpha t$ of all $Y_{i,2-\gamma_1}$ should be $1$.

We now bound the two terms in the final inequality separately.

For the second term, if $\max_{j \in R_c}\Lambda_j \geq \gamma_1\Delta$. 

\begin{align*}
    \Pr[\max_{j \in R_c}\Lambda_j\geq \gamma_1\Delta]\leq Q_c \Pr[Y_{j,\gamma_1}=1]\leq Q_c (\frac{p}{m})^{\gamma_1}
\end{align*}
Here, we use union bound over all elements in $R_c$ for the first inequality and the second inequality is plugging in the value of $\Pr[Y_{j,\gamma_1} = 1]$, which we have already computed.

Now, we need to provide a bound on $Q_c$. Note that if $Q_c$ denotes the correct number of nodes, which  corresponds to the majority of nodes, then $Q_c \leq (1 - \alpha) S_c$, where $S_c$ is the size of the cluster $c$.

For the first term, we can use Chernoff bound as $Y_{i,2-\gamma_1}$ are independent random variables with expectation $\frac{p}{m}$
\begin{align*}
    \Pr[\frac{1}{m}\sum_{i=1}^m Y_{i,2-\gamma_1} \geq \alpha \frac{t}{m}] \leq \exp(-(\alpha \frac{t}{m} - \E[Y_{i,2-\gamma_1}])^2 m)\leq \exp(-a_2 m)
\end{align*}
We need $\alpha\frac{t}{m} \geq \E[Y_{i,2-\gamma_1}]$,which implies  $\alpha t \geq 1$, since $Y_{i,2-\gamma_1}$ is a bernoulli random variable. Further, we require $\alpha t = \Theta(m)$, so that we can bound the probability using a constant $a_2\geq0$.
If we choose $\gamma_1$ as a constant independent of $m$, then we are done.

Now, plugging all these inequalities into Eq~\eqref{eq:alpha_proof_int}, we get
\begin{align*}
    &\Pr[\exists c \in \range(\mathcal{C}_0)  \text{ which has $\geq \alpha$ fraction of wrong nodes}]\\
    &\quad\leq \range(\mathcal{C}_0)\exp(-a_2 m) + \sum_{c \in \range(\mathcal{C}_0)}(1-\alpha)S_c(\frac{p}{m})^{\gamma_1}\\
    &\quad\leq \abs{\range(\mathcal{C}_0)}\exp(-a_2 m) + (1-\alpha)m(\frac{p}{m})^{\gamma_1}\\
    &\quad \leq \frac{m}{t}\exp(-a_2 m) + (1-\alpha)m(\frac{p}{m})^{\gamma_1}\\
\end{align*} 
For the second inequality, we use $\sum_{c\in \mathcal{C}_0}S_c = m$ and for the third inequality, we use $\abs{\range(\mathcal{C}_0)}t\leq m$.

\subsection{Proof of Lemma~\ref{lem:merge_error}}\label{sec:merge}

First, let $i,j\in[m]$ be a node in cluster  $c, c' \in \range(\mathcal{C}_0)$ respectively such that $\mathcal{C}^\star(j)$ and $\mathcal{C}^\star(i)$ are the cluster labels of clusters $c$ and $c'$ respectively. Then, if we repeat our thresholding analysis for \texttt{MERGE()} operation, we obtain

\begin{align*}
    &\mathsf{dist}(w_i^\star,w_j^\star) - \Psi_{c,c'}\leq\mathsf{dist}(\omega_{c,T}, \omega_{c',T}) \leq \mathsf{dist}(w_i^\star,w_j^\star) + \Psi_{c,c'}\\
    &\text{where } \Psi_{c,c'} = \mathsf{dist}(\omega_c^\star,w_i^\star) + \mathsf{dist}(\omega_{c'}^\star, w_j^\star) + \sum_{k = c, c'}\mathsf{dist}(w_{k,T},w_{k}^\star)
\end{align*}
We obtain the above equations by a simple application of triangle inequality. Here, $\omega_c^\star$ is as defined in Appendix~\ref{sec:preliminaries}.

To analyze the above quantities, we need to bound $\norm{\omega_c^\star - \omega_{c,T}}$ and $\norm{\omega_c^\star - w_j^\star}$ for some $j\in G_{c}$.
The following Lemmas provide these bounds.
\begin{lemma}[Convergence of $\omega_{c,T}$]\label{lem:cluster_conv}
If Assumptions~\ref{assumption:str_cvx},\ref{assumption:smooth} and ~\ref{assumption:lipschitz} hold, and $\eta \leq \frac{1}{L}$, then
    \begin{equation}
        \norm{\omega_{c,T} - \omega_c^\star} \leq (1-\kappa^{-1})^{T/2}D + \frac{2}{\mu}\Lambda_c \quad \forall c \in \range(\mathcal{C}_0)    
    \end{equation}
    where  $\kappa = \frac{L}{\mu}$ and $\Lambda_c$ is a positive random variable with
    \begin{align}
        &\Pr[\Lambda_c \geq \sqrt{2d}\frac{r + 3\beta s}{1-2\beta} + \sqrt{2}\frac{2(1 + 3\beta)}{1 - 2\beta}\delta \Hat{L}]\nonumber\\
        &\leq 2d(1 + \frac{D}{\delta})^d\biggl(\exp(-(1-\alpha) S_c n \min\{\frac{r}{2\Hat{L}},\frac{r^2}{2\Hat{L}^2}\}) + (1-\alpha)S_c\exp(-n\min\{\frac{s}{2\Hat{L}},\frac{s^2}{2\Hat{L}^2}\})\biggr)
    \end{align}
    for some $r,s,\delta>0$ where $S_c$ is the size of cluster $c$.
\end{lemma}
Proof is presented in Section~\ref{sec:cluster_conv_proof}

\begin{lemma}[Distance between cluster minima and node minima]\label{lem:cluster_minima}
If Assumptions~\ref{assumption:str_cvx} and ~\ref{assumption:lipschitz} are satisfied then,  for all $j\in [m]$, where $j$ is a node in cluster $c\in \mathcal{C}_0$ where $\mathcal{C}^\star(j)$ is the cluster label of node $c$, we have
\begin{equation}
    \norm{\omega_c^\star -w_j^\star}\leq\sqrt{\frac{2\Hat{L}\epsilon_1}{\mu}}:=B
\end{equation}
\end{lemma}
Proof is presented in Section~\ref{sec:cluster_minima_proof}.

Now, that we have our required quantities, we are ready to analyze the probability of error after the merge and reclustering operations.

First, we analyze the probabilty of \texttt{MERGE()} operation.
Note that if correct nodes of $c$ and $c'$ were from the same cluster $\mathcal{C}^\star$ then, $\norm{w_i^\star - w_j^\star}\leq \epsilon_1, \forall i \in G_c, j \in G_{c'}$.
If correct nodes of $c'$ and $c$ were from different clusters in $\mathcal{C}^\star$, then, $\norm{w_i^\star - w_j^\star}\geq \epsilon_2, \forall i \in G_c, j \in G_{c'}$.
Therefore, the probability of \texttt{MERGE()} error is upper bounded by 
\begin{align}
    \Pr[\text{\texttt{MERGE()} Error}] \leq& \Pr[\text{at least 1 edge is incorrect}]\\
    \leq& \Pr[\max_{c,c'} \Psi_{c,c'} \geq \Delta_{\lambda}]\\
    \leq& \Pr[\max_{c,c'} \sum_{k = c,c'} \frac{2\Lambda_k}{\mu}\geq \Delta_{\lambda} - 2(1 - \kappa^{-1})^{T/2}D - 2B]\\
    \leq& \max_{c \in \range(\mathcal{C}_0)} \Pr[ \Lambda_c\geq \frac{\mu}{2}(\frac{\Delta_{\lambda}}{2} - (1 - \kappa^{-1})^{T/2}D - B)]\\
    \leq& \max_{c \in \range(\mathcal{C}_0)} \Pr[ \Lambda_c\geq \Delta']\label{eq:merge_proof_int}\\
    \leq& \max_{c \in \range(\mathcal{C}_0)} 4d\exp(-a_3' n \frac{\Delta'}{2\Hat{L}})\\
    \leq& \sum_{c \in \range(\mathcal{C}_0)} 4d\exp(-a_3' n \frac{\Delta'}{2\Hat{L}})\leq \frac{4dm}{t} \exp(-a_3' n\frac{\Delta'}{2\Hat{L}})\label{eq:merge_error_int}
\end{align}
For the second inequality, we expand all the terms of $\Phi_{c,c'}$.
We set $\Delta' = \frac{\mu}{2}(\frac{\Delta_{\lambda}}{2} - (1 - \kappa^{-1})^{T/2}D - B)$. Then, we set $r = \Theta(\Hat{L}\max\{\frac{\Delta'}{S_c\sqrt{d}\Hat{L}},\sqrt{\frac{\Delta'}{S_c\sqrt{d}\Hat{L}}}\}), s =\Theta(\Hat{L}\max\{\frac{\Delta'}{S_c\sqrt{d}\Hat{L}} + \frac{2\log(S_c)}{n},\sqrt{\frac{\Delta'}{S_c\sqrt{d}\Hat{L}} + \frac{2\log(S_c)}{n}}\}), \delta = \Theta(\frac{D d^{3/2} \Hat{L}}{n\Delta'})$ and if $d = \Omega(\min\{\frac{n^{2/3}\Delta^{4/3}}{D^{2/3}\Hat{L}^{2/3}},\frac{n^2\Delta'^2}{\Hat{L}^2\log(c_{\min})}\})$, such that $\sqrt{2d}\frac{r + 3\beta s}{1-2\beta} + \sqrt{2}\frac{2(1 + 3\beta)}{1 - 2\beta}\delta \Hat{L} \geq \Delta'$, then there exist some constant $a_3' >0$ such that the second inequality is satisfied by Lemma~\ref{lem:cluster_conv}. We then use the union bound, followed by $\abs{\range(\mathcal{C}_0)}\leq \frac{m}{t}$.

\subsection{Proof of Lemma~\ref{lem:recluster}}\label{sec:recluster}
We can apply our thresholding analysis to $\norm{\omega_{c,T} - w_{i,T}}$ for $c\in \range(\mathcal{C}_0)$. 
First, let $j$ be a node in cluster $c$ such that $\mathcal{C}^\star(j)$ is the cluster label of $c$.
\begin{align*}
&    \mathsf{dist}(w_j^\star, w_i^\star) + \Phi_{c,i}    \leq \mathsf{dist}(\omega_{c,T} ,w_{i,T}) \leq \mathsf{dist}(w_j^\star,w_i^\star) + \Phi_{c,i}\\
&\text{where } \Phi_{c,i} = \mathsf{dist}(\omega_{c,T},\omega_c^\star) + \mathsf{dist}(\omega_c^\star,w_j^\star) + \mathsf{dist}(w_{i,T},w_i^\star)
\end{align*}
From Appendix~\ref{sec:init_proof} and ~\ref{sec:merge}, we have bounds for all the terms involved.
Note that after merging, each cluster in $\mathcal{C}^\star$ should have only 1 cluster in $\mathcal{C}_1$.
Therefore, after we recluster according to $\norm{\omega_{c,T} - w_{i,T}}$, we incur an error if $i$ goes to the wrong cluster.
Suppose that the $c$ corresponds to the correct cluster for $i$ and $c'$ is the cluster to which it is assigned , with $c,c'\in \range(\mathcal{C}_1), c\neq c'$.
Then,
\begin{align}
    \Pr[\text{Reclustering Error}]\leq& \Pr[\max_{ i \in [m]} \max_{c' \neq c} \norm{\omega_{c',T} - w_{i,T}} \leq \norm{\omega_{c,T} - w_{i,T}}]\\
    \leq&  \Pr[\max_{ i \in [m]} \max_{c' \neq c} \epsilon_2 - \Phi_{c',i} \leq \epsilon_1 + \Phi_{c,i}]\\
    \leq&  \Pr[\max_{ i \in [m]} \max_{c' \in \mathcal{C}_0'} \Phi_{c,i} \geq \frac{\epsilon_2 - \epsilon_1}{2}]\\
    \leq&  \Pr[\max_{ i \in [m]} \max_{c' \in \mathcal{C}_0'}(\Lambda_c + \Lambda_i) \geq \Delta + \Delta']\label{eq:recluster_int}\\
    \leq&  \Pr[\max_{c \in  \mathcal{C}_0'} \Lambda_c \geq \Delta'  - (\gamma_2-1)\Delta] + \Pr[\max_{i\in[m]}\Lambda_i \geq \gamma_2\Delta]\\
    \leq&  \max_{c \in \range(\mathcal{C}_0)'} \Pr[ \Lambda_c \geq \Delta''] + \max_{i \in m}\Pr[\Lambda_i \geq \gamma_2\Delta]\label{eq:recluster_error_int}
\end{align}
For the second inequality, we use the thresholding analysis on $\norm{\omega_{c,T} - w_{i,T}}$.
For the third inequality, we rearrange the terms and combine max over $c'\neq c$ with $c$, and use. For the fourth inequality, we expand the terms of $\Phi_{c,T}$ and  substitute the values of $\Delta$ and $\Delta'$, using the inequality $\Delta_{\lambda} \leq \frac{\epsilon_2 - \epsilon_1}{2}$.
For the fifth inequality, we use consider some $\gamma_2\in (1,2 - \frac{\mu B}{2\Delta})$ and break the terms using union bound such that $\Delta'' = \Delta' - (\gamma_2 -1)\Delta \geq 0$.
Finally, we use the union bound on $c \in \range(\mathcal{C}_0)'$ and $i\in[m]$.

Now, we bound the two terms in Eq~\eqref{eq:recluster_error_int} separately.
The second term can be bounded in terms of $Y_{i,\gamma_2}$.
Thus,
\begin{align}
    \max_{i\in[m]}\Pr[\Lambda_i \geq \gamma_2\Delta]= \max_{i \in [m]}\Pr[Y_{i,\gamma_2} =1]\leq m (\frac{p}{m})^{\gamma_2}
\end{align}
We use expectation of $Y_{i,\gamma_2}$ calculated in Appendix~\ref{sec:alpha_fraction_proof} and then bound max by sum.

For the first term, our analysis is similar to that of \texttt{MERGE()} error. Assume that there is some constant $u_2 > 1$ such that $\Delta'' \geq u_2\Delta'$. We set $r = \Theta(\Hat{L}\max\{\frac{\Delta'}{S_c\sqrt{d}\Hat{L}},\sqrt{\frac{\Delta'}{S_c\sqrt{d}\Hat{L}}}\}), s =\Theta(\Hat{L}\max\{\frac{\Delta'}{S_c\sqrt{d}\Hat{L}} + \frac{2\log(S_c)}{n},\sqrt{\frac{\Delta'}{S_c\sqrt{d}\Hat{L}} + \frac{2\log(S_c)}{n}}\}), \delta = \Theta(\frac{D d^{3/2} \Hat{L}}{n\Delta'})$ and if $d = \Omega(\min\{\frac{n^{2/3}\Delta^{4/3}}{D^{2/3}\Hat{L}^{2/3}},\frac{n^2\Delta'^2}{\Hat{L}^2\log(c_{\min})}\})$, such that $\sqrt{2d}\frac{r + 3\beta s}{1-2\beta} + \sqrt{2}\frac{2(1 + 3\beta)}{1 - 2\beta}\delta \Hat{L} \geq \Delta'$, then there exist some constant $a_3'' >0$ such that the second inequality is satisfied by Lemma~\ref{lem:cluster_conv}. We then use the union bound, followed by $\abs{\range(\mathcal{C}_0)}\leq \frac{m}{t}$.

\begin{align}
\max_{c \in \range(\mathcal{C}_0)'} \Pr[ \Lambda_c \geq \Delta''] \leq& \max_{c \in \range(\mathcal{C}_0)'} 4d \exp(-a_3'' n \frac{\Delta'}{2\Hat{L}})\\
    \leq& \sum_{c \in \range(\mathcal{C}_0)'} 4d \exp(-a_3'' n \frac{\Delta'}{2\Hat{L}})\\
    \leq&  \frac{4d m}{t} \exp(-a_3'' n \frac{\Delta'}{2\Hat{L}})
\end{align}

\subsection{Proof of Lemma~\ref{lem:cluster_conv}}
\label{sec:cluster_conv_proof}
First, we use an intermediate Lemma from ~\cite{pmlr-v80-yin18a}. This characterizes the behavior of $TrimmedMean_{\beta}$
gradient estimator.
\begin{lemma}[TrimmedMean Estimator Variance]\label{lem:trmean}
    Let $g_c(w)$ be the output of  $\mathrm{TrMean}_{\beta}$ estimator for cluster $c\in \mathcal{C}_0$ with size of cluster $S_c$. If Assumptions~\ref{assumption:lipschitz} holds, then
    \begin{equation}
    \begin{aligned}
        &\norm{ g_c(w) - \nabla F_c(w)}  \leq \Lambda \\
        &\text{where } \Pr[\Lambda \geq \sqrt{2d}\frac{r + 3\beta s}{1-2\beta} + \sqrt{2}\frac{2(1 + 3\beta)}{1 - 2\beta}\delta \Hat{L}]\\
        &\leq 2d(1 + \frac{D}{\delta})^d\biggl(\exp(-(1-\alpha) S_c n \min\{\frac{r}{2\Hat{L}},\frac{r^2}{2\Hat{L}^2}\}) + (1-\alpha)S_c\exp(-n\min\{\frac{s}{2\Hat{L}},\frac{s^2}{2\Hat{L}^2}\})\biggr)
    \end{aligned}
    \end{equation}
    for some $r,s,\delta>0$.
\end{lemma}

\begin{proof}
The proof of this Lemma follows from coordinate-wise sub-exponential distribution of $\nabla F_c$.
Since loss per sample $f(w,z)$ is Lipschitz in each of its coordinates with Lipschitz constant $L_k$ for $k \in[d]$. Thus, $F_c(w)$ is also $L_k$-Lipschitz for  each coordinate $k\in[d]$ from Corrolary~\ref{corr:sum_lipschitz}.
Now, every subgaussian variable with variance $\sigma^2$ is $\sigma$-sub exponential. Thus, each coordinate of $\nabla_w f(w,z)$ is $\Hat{L}$-sub-exponential, since $\Hat{L} > L_k ,\forall k \in [d]$. 
The remainder of proof can be found in ~\citep[Appendix~E.1]{pmlr-v80-yin18a}.
\end{proof}

Now, using the above Lemma, we can bound the iterate error for a cluster $c\in \mathcal{C}_0$.
Consider $\norm{\omega_{c,t+1} - \omega_c^\star}^2$,
\begin{align*}
    \norm{\omega_{c,t+1} - \omega_c^\star} \leq& \norm{proj_{\cW}\{\omega_{c,t} - \eta\nabla g(\omega_{c,t})\} - \omega_c^{\star}}\\
    \leq& \norm{\omega_{c,t} - \eta \nabla g(\omega_{c,t}) - \omega_c^{\star}}\\
    \leq& \norm{\omega_{c,t} - \eta\nabla F(\omega_{c,t}) - \omega_c^{\star}} + \eta\norm{g(\omega_{c,t}) - \nabla F(\omega_{c,t})}\\
    \leq& \norm{\omega_{c,t} - \eta\nabla F(\omega_{c,t}) - \omega_c^{\star}} + \eta\Lambda
\end{align*}
Now, we bound $\norm{\omega_{c,t} - \eta \nabla F(\omega_{c,t}) - \omega_c^{\star}}^2$ using $\mu$-strong convexity and $L$-smoothness of $F_c$. 
The analysis is similar to the convergence analysis in Section~\ref{sec:conv_node_proof}. Thus, for $\eta \leq \frac{1}{L}$
\begin{align*}
    \norm{\omega_{c,t} - \eta \nabla F(\omega_{c,t}) - \omega_c^{\star}}^2 \leq (1-\eta\mu) \norm{ \omega_{c,t} - \omega_c^\star}^2
\end{align*}
Using this bound we can analyze the original term with $\norm{\omega_{c,t+1} - \omega_c^\star}$.
\begin{align*}
    \norm{\omega_{c,t+1} - \omega_c^\star} \leq& \sqrt{1-\eta\mu} \norm{ \omega_{c,t} - \omega_c^\star} + \eta\Lambda\\
    \norm{\omega_{c,T} - \omega_c^\star}\leq& (1-\eta\mu)^{T/2} \norm{ \omega_{c,0} - \omega_c^\star} + \eta\Lambda(\sum_{t=0}^{T-1}(1 - \eta\mu)^{t/2})\\
    \leq& (1-\kappa^{-1})^{T/2} \norm{ \omega_{c,0} - \omega_c^\star} + \eta\Lambda(\sum_{t=0}^{\infty}(1 - \frac{\eta\mu}{2})^{t})\\
    \leq& (1-\kappa^{-1})^{T/2}D + \frac{2}{\mu}\Lambda
\end{align*}
For the second inequality, we use $\kappa = \frac{L}{\mu}$ and unroll the recursion for $T$ steps. 
For the third inequality, we use $\sqrt{1 - x} \leq 1 - \frac{x}{2}$ and upper bound the finite geometric sum by its infinite counterpart. 
Finally we use the boundedness of $\cW$ and the sum of the geometric series to get our result.

\subsection{Proof of Lemma~\ref{lem:conv_node_int}}
\label{sec:conv_node_int_proof}
    We present the proof for this lemma here as it is a corollary of Lemma~\ref{lem:cluster_conv}.
    
    We utilize the intermediate Lemma~\ref{lem:trmean}. Now, if we set $\alpha = \beta = 0$ and $S_c = 1$, we obtain the generalization guarantee
    for GD on a single node $i\in [m]$. Further, we do not need the terms of $s$ as they appear with $\beta$, and thus, we can choose $s$ very large, so that 
    we can ignore its contribution to error probability.
    The remainder of the proof follows that of Lemma~\ref{lem:cluster_conv}.
\subsection{Proof of Lemma~\ref{lem:cluster_minima}}
\label{sec:cluster_minima_proof}

Since $F_c$ is $\Hat{L}$-Lipshchitz and $\mu$-strongly convex with minima $\omega_c^\star$,
\begin{align*}
    F_c(w_i^\star) - F_c(\omega_c^\star) =& \frac{F_i(w_i^\star) - F_i(\omega_c^\star)}{Q_c} +  \sum_{j\neq i, \mathcal{C}_0(j) = c}\frac{F_j(w_i^\star) - F_j(\omega_c^\star)}{Q_c}\\
    \leq& \frac{F_i(w_i^\star) - F_i(\omega_c^\star)}{Q_c} +  \sum_{j\neq i, \mathcal{C}_0(j) = c}\frac{F_j(w_i^\star) - F_j(w_j^\star)}{Q_c}\\
    \leq& -\frac{\mu\norm{w_i^\star - \omega_c^\star}^2}{2Q_c} +  \sum_{j\neq i, \mathcal{C}_0(j) = c}\frac{\Hat{L}\norm{w_i^\star - w_j^\star}}{Q_c}\\
    \frac{\mu}{2}\norm{w_i^\star - \omega_c^\star}^2\leq& -\frac{\mu\norm{w_i^\star - \omega_c^\star}^2}{2Q_c} +  \frac{(Q_c - 1)\Hat{L}\epsilon_1}{Q_c}\\
    \frac{\mu}{2}\norm{w_i^\star - \omega_c^\star}^2\leq& -\frac{\mu\norm{w_i^\star - \omega_c^\star}^2}{2Q_c} +  \frac{(Q_c - 1)\Hat{L}\epsilon_1}{Q_c}\\
    \norm{w_i^\star - \omega_c^\star}^2\leq& \frac{2\Hat{L}\epsilon_1}{\mu}\\
    \norm{w_i^\star - \omega_c^\star}\leq& \sqrt{\frac{2\Hat{L}\epsilon_1}{\mu}}
\end{align*}
For the first equation, we expand $F_c$ into its component terms, where $Q_c$ denotes the number of correct nodes in cluster $c$.
For the second inequality, we use the fact that $w_j^\star = \argmin_{w \in \cW}F_j(w)$. For the third inequality, we use strong-convexity of $F_i$ and $\Hat{L}$-Lipschitzness for $F_j, j\neq i$.
For the fourth inequality, we use a lower bound on $F_c(w_i^\star) - F_c(\omega_c^\star)$ using $\mu$-strong convexity of $F_c$.
Finally, we manipulate the remaining terms to obtain the final bound.

\section{Proof of Theorem~\ref{thm:convergence}}
\label{sec:convergence_proof} By Theorem~\ref{thm:refine}, $\mathcal{C}_R  \neq \mathcal{C}^\star$, with probability $\big(\frac{\rho_2 }{m^{(1-\rho_1)}}p\big)^R$. For the $(R+1)^{th}$ step, we bound probability of error by $1$. Therefore, with probability $1 - \exp(-\frac{5}{8}R)p$. For the $(R+1)^{th}$ step, we optimize the cluster iterates from \texttt{TrimmedMeanGD()} to improve convergence instead of clustering error. 
Since $\mathcal{C}_{R+1} = \mathcal{C}_R$, each cluster in $\mathcal{C}_{R+1}$ maps to some cluster in $\mathcal{C}^\star$. 
Without loss of generality, assume that cluster $c \in \range(\mathcal{C}_{R+1})$ maps to the same cluster $c \in \mathcal{C}$. Now, if $\{c_1, c_2, \ldots, c_l\}$ are the clusters in $\mathcal{C}_{R}$ which merged to form cluster $c \in \range(\mathcal{C}_{R+1})$. Then, we can write

\begin{align}
    \norm{\omega_{c,T} - \omega_c^\star} =& \norm{\frac{1}{l}\sum_{j=1}^l (\omega_{c_j,T} - \omega_c^\star)}\\
    \leq & \frac{1}{l}\sum_{j=1}^l \norm{\omega_{c_j,T} - \omega_c^\star}\\
    \leq & \frac{1}{l}\sum_{j=1}^l (\norm{\omega_{c_j,T} - \omega_{c_j}^\star} + \norm{\omega_{c_j}^\star - \omega_c^\star})\\
\end{align}
For the first inequality, we used the definition of $\omega_{c,T}$ from \texttt{MERGE()}. For the second inequality, we used the triangle inequality for the $l$ elements. The third inequality is obtained by using triangle inequality and adding and subtracting $\omega_{c_j}^\star$ as defined in Appendix~\ref{sec:preliminaries}.

Now, consider the set of nodes $\{i_1,i_2,\ldots, i_l\} \subseteq [m]$, such that $i_j \in c_j \forall j \in [l]$ and $\mathcal{C}^\star(i_j) = c \forall j \in [l]$. Therefore, we can split each term of $\norm{\omega_{c_j}^\star - \omega_c^\star}$ as --
\begin{align}
    \norm{\omega_{c,T} - \omega_c^\star} \leq & \frac{1}{l}\sum_{j=1}^l (\norm{\omega_{c_j,T} - \omega_{c_j}^\star} + \norm{\omega_{c_j}^\star -w_{i_j}}  + \norm{w_{i_j} - \omega_c^\star})\\
    \leq & \frac{1}{l}\sum_{j=1}^l \norm{\omega_{c_j,T} - \omega_{c_j}^\star} + 2B\\
\end{align}
From Lemma~\ref{lem:cluster_minima}, since $i_j$ contributes to both clusters $c_j$ and $c^\star$, we can bound the difference from their minima by $B$. Further, we can use Lemma~\ref{lem:cluster_conv} and the Lemma~\ref{lem:trmean}, which is adapted from \citep[Theorem~4]{pmlr-v80-yin18a},to bound the convergence of $\norm{\omega_{c_j,T} - \omega_{c_j}^\star}$. If we set $\delta = \frac{1}{n S_{c_j} \Hat{L} D}$ and
\begin{align*}
    r &= \Hat{L}\max\{\frac{8d}{n S_{c_j}}\log(1 + n S_c \Hat{L} D), \sqrt{\frac{8d}{n S_{c_j}}\log(1 + n S_c \Hat{L} D)}\}\\ 
    s &= \Hat{L}\max\{\frac{4d}{n}(d\log(1 + n S_{c_j} \Hat{L} D) +\log m) , \sqrt{\frac{4d}{n}(d\log(1 + n S_{c_j} \Hat{L} D) +\log m)}\}
\end{align*}
where $S_{c_j}$ is the size of cluster $c_j$, we obtain 
\begin{align}
    \norm{\omega_{c,T} - \omega_c^\star}\leq &  (1 - \kappa^{-1})^{T/2}D+ \Lambda' + 2B\\
\end{align}
where 
\begin{align*}
\Lambda'= \cO\biggl(\frac{\Hat{L}d}{1 - 2\beta}\biggl(\frac{\beta}{\sqrt{n}} + \frac{1}{\sqrt{n c_{\min}}}\biggr)\sqrt{\log(n \max_{j\in[l]}S_{c_j}\Hat{L}D)}\biggr)
\end{align*}
We can further upper bound $\max_{j \in [l] S_{c_j}}$ by $m$.
Now, the probability of error for each cluster $c \in \range(\mathcal{C}_R)$ for given values of $r$ and $s$ is $\frac{4d}{(1 + n c_{\min}\Hat{L}D)^d}$, therefore, we can use union bound and multiply this probability of error by $\range(\mathcal{C}_R) \leq \frac{m}{t}$. Since $t  = \Theta(c_{\min})$, we can upper bound this by $\frac{m u''}{c_{\min}}$ for some positive constant $c_{\min}$.

\section{Additional Technical Lemmas}
\begin{lemma}\label{lem:sum_str_cvx}
    If $f,g:\R^d \to \R$ are two $\mu$-strongly convex functions on a domain $\cW$. Then, $\frac{f + g}{2}$ is also $\mu$-strongly convex on the same domain.
\end{lemma}
\begin{proof}
If $f$ and $g$ are $\mu$-strongly convex on a domain $\cW$, then for any $w_1, w_0 \in \cW$
\begin{align*}
    f(w_1) &\geq f(w_0) + \lin{\nabla f(w_0), w_1 - w_0} + \frac{\mu}{2}\norm{w_1 - w_0}^2\\
    g(w_1) &\geq g(w_0) + \lin{\nabla g(w_0), w_1 - w_0} + \frac{\mu}{2}\norm{w_1 - w_0}^2 
\end{align*}
Adding the above equations, we get 
\begin{align*}
    \frac{f(w_1) + g(w_1)}{2}\geq \frac{f(w_0) + g(w_0)}{2}+ \lin{\frac{\nabla f(w_0)+\nabla g(w_0)}{2}, w_1 - w_0} + \frac{\mu}{2}\norm{w_1 - w_0}^2
\end{align*}
Thus, $\frac{f+g}{2}$ is also $\mu$-strongly convex.
\end{proof}

\begin{lemma}\label{lem:sum_smooth}
    If $f,g:\R^d \to \R$ are two $L$-smooth functions on a domain $\cW$. Then, $\frac{f + g}{2}$ is also $L$-smooth on the same domain.
\end{lemma}
\begin{corollary}\label{corr:sum_lipschitz}
    If $f,g:\R^d \to \R$ are two $L$-Lipschitz functions on a domain $\cW$. Then, $\frac{f+g}{2}$ is also $L$-Lipschitz on the same domain.
\end{corollary}
\begin{proof}
Consider the following term for  any $w_1, w_0 \in \cW$
\begin{align*}
    \norm{\frac{\nabla f(w_1) + \nabla g(w_1)}{2} - \frac{\nabla f(w_0) + \nabla g(w_0)}{2}}&\leq\frac{1}{2}\norm{(\nabla f(w_1) - \nabla f(w_0)) + (\nabla g(w_1) - \nabla g(w_0))}\\
    &\leq\frac{1}{2}(\norm{\nabla f(w_1) - \nabla f(w_0)} + \norm{\nabla g(w_1) - \nabla g(w_0)})\\
    &\leq \frac{1}{2}(L\norm{w_1 - w_0} + L\norm{w_1 - w_0})\\
    &\leq L\norm{w_1 - w_0}
\end{align*}
In the second inequality, we use the triangle inequality of norms. For the third inequality, we use the $L$-smoothness of $f$ and $g$. Thus, $\frac{f + g}{2}$ is also $L$-smooth
The proof of the corollary is same as above, by replacing terms of $\nabla f$ and $\nabla g$ by $f$ and $g$ respectively. 
\end{proof}

\begin{lemma}\label{lem:coordinate_lipschitz}
If each coordinate of a function $f:\R^d \to \R$ is $L_k$-Lipschitz for $k\in[d]$ on the domain $\cW$, then $f$ is $\Hat{L} = \sqrt{\sum_{k=1}^d L_k^2}$-Lipschitz on the same domain $\cW$.
\end{lemma}
\begin{proof}

Consider $w_1,w_0 \in \cW$.Define a sequence of variables $\{w[k] = ((w_1)_1,(w_1)_2\ldots, (w_1)_k, (w_0)_{k+1},\ldots (w_0)_{d})^\intercal\}_{k=0}^d$.
Then, $w_1 = w[d]$ and $w_0 = w[0]$
\begin{align*}
    \abs{f(w_1) - f(w_0)} =&\abs{\sum_{k=1}^d (f(w[k]) -f(w[k-1]))}\\
    =&\sum_{k=1}^d L_k \abs{(w_1)_k -(w_0)_k}
\end{align*}
The second inequality follows by using triangle rule. Then, $f(w[k])$ and $f(w[k-1])$ differ only in the $k^{th}$ coordinate, so we apply $L_k$ coordinate-wise Lipschitzness.
Now, consider a random variable $v\in\R^d$ such that $v_k = L_k \frac{\abs{(w_1)_k -(w_0)_k}}{(w_1)_k -(w_0)_k}$ if $(w_1)_k -(w_0)_k \neq 0$, else $0$.
Then, 
\begin{align}
    \sum_{k=1}^d L_k \abs{(w_1)_k -(w_0)_k} =& \lin{v, w_1 - w_0}\\
    \leq& \norm{v}\norm{w_1 - w_0}\\
    \leq&\sqrt{\sum_{k=1}^d L_k^2} \norm{w_1 -w_0}
\end{align}
Here, we use the Cauchy-Schwartz inequality for the second step. Then, note that each coordinate of $v$ is bounded by $L_k$.
\end{proof}